\documentclass{article}

\usepackage{microtype}
\usepackage{graphicx}
\usepackage{subcaption}
\usepackage{booktabs} 

\usepackage{hyperref}
\usepackage{xurl}

\usepackage[preprint]{icml2026}

\usepackage{amsmath}
\usepackage{amssymb}
\usepackage{amsfonts}
\usepackage{mathtools}
\usepackage{amsthm}
\usepackage{dsfont}
\usepackage{algorithm}
\usepackage{algorithmic}

\newcommand{\one}{\mathds{1}}

\newcommand{\norm}[1]{\left\Vert#1\right\Vert}

\newcommand{\bbR}{\mathbb{R}}

\newcommand{\R}{\mathbb{R}}
\newcommand{\E}{\mathbb{E}}

\usepackage[capitalize,noabbrev]{cleveref}

\theoremstyle{plain}
\newtheorem{theorem}{Theorem}[section]

\newtheorem{lemma}[theorem]{Lemma}

\theoremstyle{definition}
\newtheorem{definition}[theorem]{Definition}
\newtheorem{assumption}[theorem]{Assumption}
\theoremstyle{remark}
\newtheorem{remark}[theorem]{Remark}

\usepackage[textwidth=1.8cm, textsize=tiny]{todonotes}
\icmltitlerunning{Demographic Parity Tails for Regression}

\begin{document}

\twocolumn[
  \icmltitle{Demographic Parity Tails for Regression
  }




  \icmlsetsymbol{equal}{*}

  \begin{icmlauthorlist}
    \icmlauthor{Sinh Le Nhat}{A1}
    \icmlauthor{Christophe Denis}{A2}
    \icmlauthor{Mohamed Hebiri}{A1}
  \end{icmlauthorlist}

    \icmlaffiliation{A1}{Universit\'e Gustave Eiffel, Paris, France}  
    \icmlaffiliation{A2}{Universit\'e  Paris 1 Panth\'eon-Sorbonne, France}

  \icmlcorrespondingauthor{Sinh Le Nhat}{nhat-sinh.le2@univ-eiffel.fr}

  \icmlkeywords{Fairness, Regression, Optimal transport, Partial fairness}

  \vskip 0.3in
]



\printAffiliationsAndNotice{}  

\begin{abstract}
Demographic parity (DP) is a widely studied fairness criterion in regression, enforcing independence between the predictions and sensitive attributes.
However, constraining the entire distribution can degrade predictive accuracy and may be unnecessary for many applications, where fairness concerns are localized to specific regions of the distribution. 
To overcome this issue, we propose a new framework for regression under DP that focuses on the tails of target distribution across sensitive groups. Our methodology builds on optimal transport theory. By enforcing fairness constraints only over targeted regions of the distribution, our approach enables more nuanced and context-sensitive interventions. Leveraging recent advances, we develop an interpretable and flexible algorithm that leverages the geometric structure of optimal transport. We provide theoretical guarantees, including risk bounds and fairness properties, and validate the method through experiments in regression settings.




\end{abstract}

\section{Introduction}
\label{sec:intro}

Fairness is now a major concern in today’s society. This is mainly due to the fact that unfair treatment often involves sensitive forms of discrimination against minorities, which must be mitigated. Algorithms have been identified as major contributors to such unfair treatment, since they can inherit or even amplify biases present in the data used to train them~\cite{calders2009building,zemel2013learning,barocas-hardt-narayanan}.
Consequently, fairness has emerged as a central topic in the machine learning community, as illustrated by the growing literature in the field~\cite{lum2016statistical,zafar2017fairness,barocas-hardt-narayanan, jiang2019wasserstein,chiappa2020general,feldman2015certifying,gordaliza2019obtaining,hardt2016equality,dwork2012}.
In the supervised setting, extensive studies have demonstrated that unfairness can be mitigated through pre-, in-, or post-processing strategies, and several theoretically grounded approaches have emerged~\citep{agarwal2019fairreg,Chzhen_Denis_Hebiri_Oneto_Pontil19,chiappa2020general,chzhen2020minimax, pmlr-v206-gaucher23a, gouic2020price,denis2024multiclassdp}. An important framework for fairness is regression under \emph{Demographic Parity (DP)}, a popular notion of fairness that requires predictions to be independent of a sensitive attribute, such as race, gender, or income level~\cite{hardt2016equality,chiappa2020general}. Such a requirement acts on the entire distribution of the prediction random variable—or, equivalently, leads to a modified prediction function—and often has the drawback of degrading overall prediction accuracy.

In the present contribution, we instead focus on local interventions. In particular, our goal is to ensure fair prediction or treatment for all individuals beyond a specific range of prediction values, regardless of the sensitive group to which they belong.
For instance, in clinical risk prediction, a hospital may rely on a model to identify high-risk patients requiring intensive monitoring.
In this context, fairness must hold among individuals assigned high predicted risk (\emph{e.g.,} above a critical risk threshold), ensuring comparable treatment across demographic groups.
Similarly, in resource allocation problems such as the prioritization of applicants for subsidized housing or financial aid, fairness may be required primarily among top-ranked individuals who are eligible to receive the resource.
As a second type of application, one might wish to control the proportion of individuals exceeding this threshold. In both contexts, this leads to imposing fairness above a threshold defined as a particular quantile of the prediction distribution, ensuring that no sensitive group is disproportionately represented among high-risk patients or selected beneficiaries.


The framework we propose can be seen as \emph{DP-tails fair prediction}. One of its main benefits is that it reduces the loss in prediction accuracy that typically arises when enforcing global fairness. This issue is commonly addressed through the framework of approximate fairness~\cite{dwork2012,agarwal2019fairreg}, which could complement our study. However, our motivation is different: we aim to incorporate fairness in a localized manner, and we believe that our framework provides a more precise understanding of the DP constraint in the regression setting.

\paragraph{Contributions.} 
The contributions of this work are threefold: \textbf{i)} We introduce the problem of  Demographic Parity \emph{tails}, where the goal is to enforce DP only above a given threshold in the space of prediction values. We consider two settings: one in which the proportion of predictions exceeding the threshold is fixed, and another in which this proportion is optimized; \textbf{ii)} Using tools from Optimal Transport (OT), we derive an explicit expression of the optimal solution to the regression problem under DP-\emph{tails} and propose a data-driven procedure based on this solution. Our algorithm acts as a post-processing step applied to any base regression function; \textbf{iii)} We provide theoretical guarantees on both the risk of our algorithm and its level of unfairness under an appropriate fairness notion. We additionally conduct experiments on synthetic and real datasets to demonstrate its effectiveness.

\paragraph{Related work.}
The work most closely related to ours is~\cite{chzhen2020fairwb,gouic2020price}, which studies regression under DP and leverages optimal transport arguments to construct both the fair target distribution and the corresponding prediction function. In the same spirit, our approach relies on reformulating the regression problem with a DP-\emph{tails} constraint as an optimal transport problem.
However, since we act only on a subset of the prediction space and on partial segments of the distributions, the connection to OT is more involved in our setting and requires more delicate manipulations of quantiles. Moreover, our formulation allows us to handle different objectives, such as enforcing fairness over a range of prediction values or over a range of quantiles of the prediction random variable. Notably, by localizing the constraint, we avoid some of the negative effects of global DP enforcement, such as deterioration in prediction accuracy.

Local enforcement of DP has recently been considered in~\cite{he2025enforcing}. There, the authors introduce a notion of fairness, called \emph{partial-DP}, defined over specific quantile intervals of the prediction distribution. Their solution is based on an in-processing approach and involves discretization of the prediction space. Our work differs in several aspects. First, our method is a post-processing approach, making it suitable when an accurate but unfair base predictor has already been trained, or when unlabeled data are readily available, since our post-processing step only requires unlabeled samples. Second, our method does not rely on discretization, as it leverages optimal transport techniques. Finally, our algorithm is supported by theoretical guarantees on both prediction risk and fairness level. Let us nonetheless mention that our notion of fairness and theirs are complementary, and coincide only in a very narrow regime of our problem.


More broadly, our work falls within the literature on regression under DP constraints, which is the best theoretically understood framework for fairness; see, \emph{e.g.,}~\cite{agarwal2019fairreg,Chzhen_Denis_Hebiri_Oneto_Pontil19,chzhen2020minimax}. We conclude by noting that fairness constraints can be considered in either an aware or unaware setting. In this work, we focus on the aware setting, meaning that the sensitive attribute is available at prediction time.

\paragraph{Notation.}
Given positive sequences $(a_n)$ and $(b_n)$, we write $a_n \lesssim b_n$ if there exists $c>0$ such that $a_n \le c\, b_n$ for all $n$.
For a finite set $\mathcal{S}$, $|\mathcal{S}|$ denotes its cardinality.
The symbol $\mathbf{supp}(\nu)$ denotes the support of a nonnegative measure~$\nu$.
For a univariate probability measure $\mu$, $F_\mu$ denotes its cumulative distribution function (CDF) and $Q_\mu:[0,1]\to\mathbb{R}$ its quantile function, defined for $t\in(0,1]$ by
$ 
Q_\mu(t)=\inf\{y\in\mathbb{R}:F_\mu(y)\ge t\}$ and $ Q_\mu(0)=Q_\mu(0+).$ Finally we set $O(\cdot)$ for the big-$O$ notation.


\section{Statistical setting and optimal rule}
\label{sec:statSetting}
In this section, we introduce the main notation and describe our strategy for provide the optimal rule that minimizes the quadratic risk under our new partial notion of fairness

\paragraph{Regression under DP-tails.}
We study the regression model given by
\begin{equation*}
Y=f^*(X, S)+\varepsilon \enspace ,
\end{equation*}
where $(X,S)$ consists of a feature vector $X$ together with a sensitive attribute $S$ and $\varepsilon$ is a noise random variable with finite second moment and such that $\mathbb{E}\left[\varepsilon|X,S\right] = 0$. The tuple $(X,S,Y)$ follows the joint distribution $\mathbb{P}$ on $\mathbb{R}^d \times \mathcal{S} \times \mathbb{R}$, where $\mathcal{S}$ is a finite set. Moreover, the function $f^* : \mathbb{R}^d \times \mathcal{S} \to \mathbb{R}$ represents the regression rule minimizing the mean squared risk defined as
\begin{equation}
\label{eq:risk}
R(f) = \mathbb{E}\left[(Y-f(X,S))^2\right]    \enspace .
\end{equation}
Given a value of the sensitive attribute $s\in \mathcal{S}$ and any regressor (predictor)  $f: \mathbb{R}^d \times \mathcal{S} \to \mathbb{R}$, let $\nu_{f|s}$ denote the conditional law of $f(X,S)$ given $S=s$ and we will abbreviate the notation by writing $F_{f|s}$ and $Q_{f|s}$ respectively for the corresponding conditional CDF and quantile function. 
To build on our theoretical analysis, we require the following regularity assumption that in particular ensures that the our optimal solution under DP-tails fairness is well defined.
\begin{assumption}[Density]
\label{ass:density}
We assume that  $\nu_{f^*|s}$ has finite second moment and admits a density for each $s\in\mathcal{S}$.
\end{assumption} 

We now turn to the concept of DP-tails fairness, which aims to reduce the effect of the sensitive attribute $S$ by enforcing fairness only on a specified portion of the distribution. In particular, we focus on the tail of the distribution and impose that the fraction of mass beyond some threshold $\alpha\in \mathbb{R}$ is equally distributed across groups.
\begin{definition}[\emph{Demographic Parity Tails}]
\label{def:PartialDPGeneral}
    For each $\alpha\in \mathbb{R}  $ and $p\in [0,1]$, a regressor $g:\mathcal{X}\times \mathcal{S}\to\mathbb{R}$ is said \textrm{$(\alpha,p)-$ DP-tails fair} if for all $s,s'\in \mathcal{S}$ 
\begin{align*}
     \left.\mathbb{P}(g(X,S)\leq t\right|S=s)=\mathbb{P}( g(X,S)\leq t|S=s')~,
    \end{align*} 
    for all $t\geq\alpha$ and if moreover, for all $s\in\mathcal{S}$ 
    $$
    F_{g|s}(\alpha)=p ~.$$ 
In the above expressions, $\alpha$ is called \textit{fairness threshold} and $p$ is called \textit{unfairness proportion}. We denote by $\mathcal{F}^p_\alpha$ the class of all $(\alpha,p)-$DP-tails fair regressors.
\end{definition} 
From the definition, we observe that the values of the CDFs $F_{g|s}$  are identical on the interval $[\alpha, +\infty)$ for every $s \in \mathcal{S}$. In the limit case where $\alpha\to-\infty$ and $p=0$, this constraint coincides to classical DP constraint that asks equality of the laws $\nu_{g|s}$ for all $s\in\mathcal{S}$. The second condition is motivated by two aspects. First, it allows to calibrate the proportion of fair predictions (\emph{i.e.,} $1-p$) which is suitable for some applications. In addition, from the technical perspective, this condition allows to explicit the form of an optimal rule.

Our central problem is to find an optimal $(\alpha,p)-$DP-tails fair predictor that minimizes the following problem 
\begin{equation*}
   \inf _{g \in \mathcal{F}^p_\alpha}~ \mathbb{E}|Y-g(X, S)|^2 \enspace. 
 \end{equation*} 
Obviously, since the noise $\varepsilon$ is such that  $\mathbb{E}\left[\varepsilon|X,S\right] = 0$, the above problem is equivalent to solving

\begin{equation*}
    (\mathcal{P}):~\inf _{g \in \mathcal{F}^p_\alpha}\mathbb{E}\left[(f^*(X,S)-g(X,S))^2\right]\enspace.
\end{equation*} 
This step is trivial but is fundamental in our methodology. In particular, it highlights the central role of Bayes predictor $f^*$ in order to build our optimal fair predictor. Moreover, this is the first step to establish the connection between the regression under DP-tails  and solving an OT problem. 

To simplify the notation, for any regressor/predictor $g$,   we introduce the shorthand
\begin{align*}
    \mathcal{E}(g):=\mathbb{E}\left[(f^*(X,S)-g(X,S))^2\right]\enspace.
\end{align*}


\paragraph{Characterization of the optimal DP-tails fair predictor.}
In this section, we introduce our strategy to build the optimal solution for problem $(\mathcal{P})$.
Our analysis relies on the concept of optimal transport between probability measures. For completeness, we briefly recall the definition of the  $2-$Wasserstein distance.
\begin{definition}[\emph{$2-$Wasserstein distance}]
    Let $\mu$ and $\nu$ be two  probability measures on $\mathbb{R}$ which have finite second moment and let $\Pi(\mu, \nu)$ denote the set of joint probability measures on $\mathbb{R} \times \mathbb{R}$ with marginals given by $\mu$ and $\nu$. The $2-$Wasserstein distance between $\mu$ and $\nu$ is defined as
\begin{eqnarray*}
 \mathcal{W}_2^2(\mu, \nu)  := \min _{\pi \in \Pi(\mu, \nu)} \int|x-y|^2 d \pi(x, y) \enspace .
\end{eqnarray*}
\end{definition}
Importantly, the characterization of the optimal solution for $(\mathcal{P})$ 
can be decoupled into two parts. Indeed, recall Definition~\ref{def:PartialDPGeneral} and let  $(\alpha,p)\in \mathbb{R} \times  [0,1]$. One can then study the expression of the solution by considering the part on $(-\infty,\alpha]$ where we do not enforce any constraint, and the other on $(\alpha,+\infty)$ where we have to handle the fairness constraint. 
This motivates the introduction of $\mathcal{M}^p_\alpha$, the set of all family of distributions $\left(\nu^++\nu^-_s\right)_{s\in\mathcal{S}}$ such that 
\begin{enumerate}
   \item[(i)] $\nu^+$ is a positive measure on $\mathbb{R}$ such that $\textbf{supp}(\nu^+)\subset (\alpha,+\infty)$,
   \item[(ii)]  for all $s\in \mathcal{S}$,
   $\nu^-_s$ is a positive measure on $\mathbb{R}$ such that $\textbf{supp}(\nu^-_s)\subset (-\infty,\alpha]$,
   \item[(iii)] For all $s\in \mathcal{S}$, $\nu^+(\mathbb{R})+\nu_s^-(\mathbb{R})=1$ and $\nu_s^-(\mathbb{R})=p.$
 \end{enumerate} 
 The last condition on the set $\mathcal{M}^p_\alpha$ is designed to ensure that $F_{g|s}(\alpha)=p$ for all $s\in \mathcal{S}$, as required by Definition~\ref{def:PartialDPGeneral}.
We are now ready to establish a connection between the optimal DP-tails fair predictors and the Wasserstein barycenter problem.
\begin{theorem}\label{BaryCenter}
Under Assumption~\ref{ass:density}, we have
\begin{align*}
     \inf_{g\in \mathcal{F}^p_\alpha}\mathcal{E}(g) 
    = \inf_{\left(\nu^++\nu^-_s\right)_{s\in\mathcal{S}}\in \mathcal{M}_\alpha^p}\sum_{s\in \mathcal{S}}p_s\mathcal{W}_2^2(\nu_{f^*|s},\nu^++\nu^-_s),
\end{align*} 
where $p_s=\mathbb{P}(S=s)$ for each $s\in\mathcal{S}$ are the groups proportion.  
\end{theorem} 
In line with prior studies about fairness regression in \cite{chzhen2020fairwb}, the problem $(\mathcal{P})$ can be express as a Wasserstein barycenter problem --- though in the different fashion. In particular, this result highlights the atypical characterization of the optimal solution though the space of measures $\mathcal{M}^p_\alpha$ and requires a careful study on the space of quantiles. Indeed, the expression of the optimal predictor $g^*_{\alpha,p}$ --- provided below in Theorem~\ref{thr:optimalPredict} --- highly relies on 
the position of the threshold $\alpha$ \emph{w.r.t.} the average quantile $\sum_{s\in\mathcal{S}}p_sQ_{f^*|s}(p)$.
Then we introduce the following two key quantities define for all $s \in \mathcal{S}$, and $\xi >0$:
$$Q^*_s(t):=\left\{\begin{array}{ll}
     \min\left\{\alpha,  Q_{f^*|s}(t)\right\} & \text{if $t\in[0,p]$}, \\
       \max\left\{\alpha,\sum\limits_{s\in\mathcal{S}}p_sQ_{f^*|s}(t)\right\}& \text{if $t\in(p,1]$};
  \end{array}\right.
  $$
  $$
  Q^\xi_s(t):=\left\{\begin{array}{ll}
     \min\left\{\alpha,  Q_{f^*|s}(t)\right\} & \text{if $t\in[0,p]$}, \\
       \max\left\{\alpha+\xi,\sum\limits_{s\in\mathcal{S}}p_sQ_{f^*|s}(t)\right\} 
       & \text{if $t\in(p,1]$}.
  \end{array}\right.
  $$ 
%
%
with $p_s=\mathbb{P}(S=s)$. We can now state our main result.

\begin{theorem}[Optimal DP-tails fair prediction]
\label{thr:optimalPredict} 
Under Assumption~\ref{ass:density}, the problem $(\mathcal{P})$ attains its infimum at $g^*_{\alpha,p}$ with 
$$
g^*_{\alpha,p}(x,s)=Q^*_s\circ F_{f^*|s}\circ f^*(x,s) \enspace,
$$ 
for all $(x.s)\in\mathbb{R}^d\times \mathcal{S}$ and the Wasserstein barycenter problem attains its infimum at the family of conditional distributions of $\left(\nu_{g^*_{\alpha,p}|s}\right)_{s\in\mathcal{S}}$.
\\ 
Moreover, if $\alpha\leq \sum_{s\in\mathcal{S}}p_sQ_{f^*|s}(p)$ then $g^*_{\alpha,p}\in\mathcal{F}^p_\alpha$, that means the problem $(\mathcal{P})$ admits an optimal solution.  
Otherwise, the problem $(\mathcal{P})$ does not admit minimum.

On the other hand, let $\xi>0$ and define $$
g^{*,\xi}_{\alpha,p}(x,s)=Q^\xi_s\circ F_{f^*|s}\circ f^*(x,s)\enspace,
$$ 
for all $(x.s)\in\mathbb{R}^d\times \mathcal{S}$. Then $g^{*,\xi}_{\alpha,p} \in \mathcal{F}^p_\alpha$, and
\begin{align*}
\left|\mathcal{E}\left(g^{*,\xi}_{\alpha,p}\right) 
-\mathcal{E}\left(g^{*}_{\alpha,p}\right) \right|= O(\xi)\quad (\text{as $\xi\to 0$}).
\end{align*} 
\end{theorem}

Theorem~\ref{thr:optimalPredict} highlights the importance of the condition $\alpha\leq \sum_{s\in\mathcal{S}}p_sQ_{f^*|s}(p)$. When it is satisfied, Problem~$(\mathcal{P})$ can be solved explicitly and it is possible to construct a sequence of predictors achieving the minimum of the $L_2-$ risk. In particular, the solution looks like what we actually expect:
Figure~\ref{fig:enter-label1} illustrates the somehow simplest situation where we observe that enforcing DP-tails fairness only relies on thresholding the quantiles $Q_{f^*|s}(\cdot)$ on the segment $[0, p]$ and on averaging the quantiles on $(p,1]$.

In contrast, when $\alpha > \sum_{s\in\mathcal{S}}p_sQ_{f^*|s}(p)$, Problem~$(\mathcal{P})$ does not always admit a minimizer, however, it is still possible to construct a sequence of predictors whose $L_2-$ risk converges to the infimum. Figure~\ref{fig:e1} explains what happens in this case. Indeed, since we have the condition $F_{g|s}(\alpha)=p $ in Definition~\ref{def:PartialDPGeneral}, we need to threshold from below the average $\sum_{s\in\mathcal{S}}p_sQ_{f^*|s}(t)$ for $t>p$ and to enforce a small correction $\xi$ so that we do not change the value of the quantile of the optimal solution at point $p$.

\begin{figure}[h]
        \centering
        \includegraphics[width=0.5\linewidth]{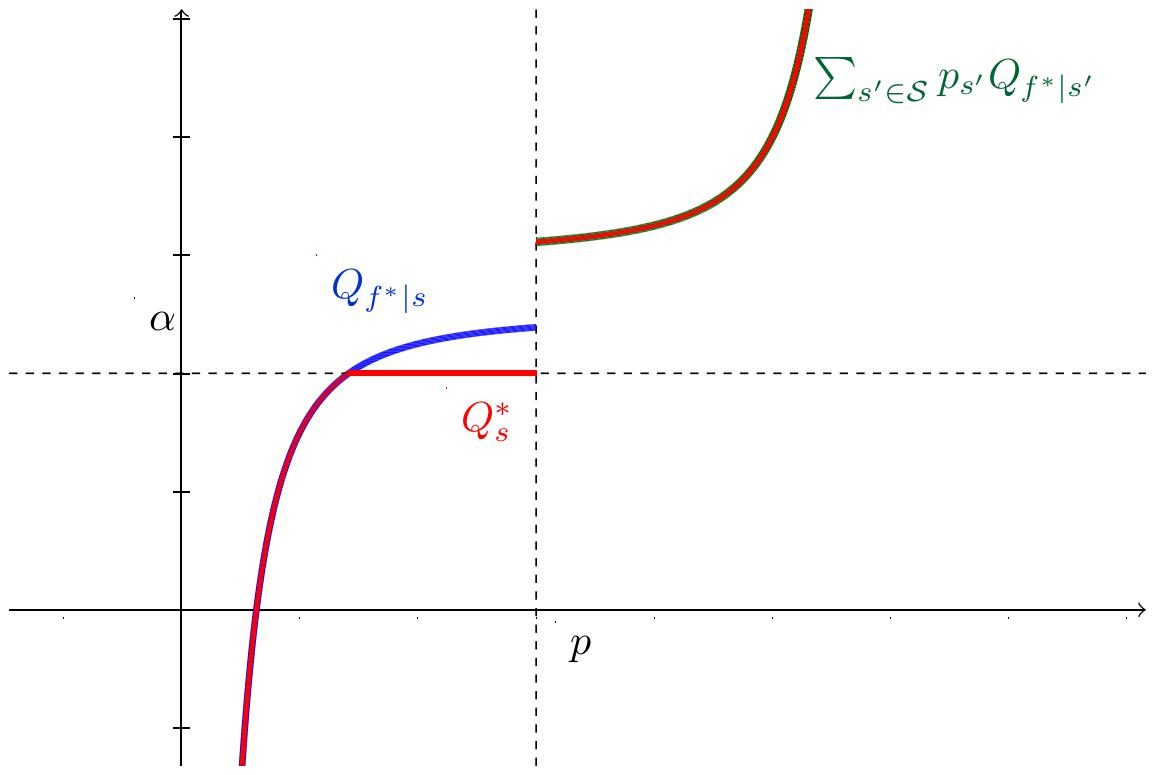}
        \caption{Illustration for $t\mapsto Q^*_s(t)$ for some $s\in \mathcal{S}$ in the case where $\alpha\leq \sum\limits_{s\in\mathcal{S}}p_sQ_{f^*|s}(p)$ (red curve).}
        \label{fig:enter-label1}
    \end{figure}

        \begin{figure}[h]
            \centering
            \includegraphics[width=0.5\linewidth]{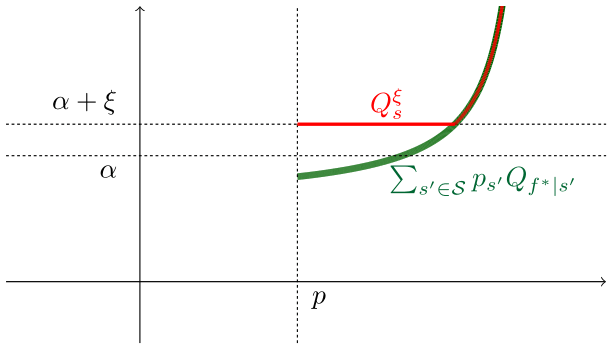}
            \caption{Illustration for $t\mapsto Q^{\xi}_s(t)$ on $(p,1]$ in the case where $\alpha >  \sum\limits_{s\in\mathcal{S}}p_sQ_{f^*|s}(p)$ (red curve).}
            \label{fig:e1}
        \end{figure}
These explanations are in line with the formal proofs of Theorems~\ref{BaryCenter} and~\ref{thr:optimalPredict} provided in Appendix~\ref{app:OptimalRules}. 
\begin{remark}
\label{erisk bound}
  In our previous discussion, we introduced two oracles, namely $g^*_{\alpha,p}$ and $g^{*,\xi}_{\alpha,p}$. The former, $g^*_{\alpha,p}$ has the drawback of not being properly defined when $\alpha > \sum_{s\in\mathcal{S}}p_sQ_{f^*|s}(p)$. The latter, $g^{*,\xi}_{\alpha,p}$, is always well defined but is not optimal when $\alpha \leq \sum_{s\in\mathcal{S}}p_sQ_{f^*|s}(p)$. Importantly,
  we can show that for any $\xi>0$ we have 
  \begin{align}
  \label{eq:eqExcessRiskXi}
    \mathbb{E}\left|g^*_{\alpha,p}(X,S)-g^{*,\xi}_{\alpha,p}(X,S)\right|^2  \leq \xi^2 \enspace.
  \end{align} 
\end{remark}
The above remark is at the core of the definition of the proposed data-driven algorithm described in Section~\ref{sec:datadriven}.
Indeed, the predictor $g^{*,\xi}_{\alpha,p}$ satisfies the DP-tails constraints, that is, $g^{*,\xi}_{\alpha,p} \in \mathcal{F}_{\alpha}^p$, and, as shown by Equation~\eqref{eq:eqExcessRiskXi}, is a good proxy for the optimal prediction in terms of risk.
\begin{remark}\label{re:inf}
The infimum value of problem $(\mathcal{P})$ can be expressed as
\[
\mathcal{E}(g^*_{\alpha,p})
= \sum_{s\in\mathcal{S}} p_s \int_0^1 \bigl(Q_{f^*|s}(x)-Q^*_s(x)\bigr)^2 \, dx.
\]
We refer to the proof of Theorem~\ref{thr:optimalPredict} in Appendix~\ref{app:OptimalRules} for details. 
This representation is useful for determining the optimal unfairness proportion~$p$ in Section~\ref{sec:optimal}.
\end{remark}

\section{Estimation method and theoretical analysis}
\label{sec:Estimation}
The previous section pointed out that $g^{*,\xi}_{\alpha,p}$, provided in Theorem~\ref{thr:optimalPredict}, is an almost optimal solution for the regression under DP-tails constraint. It has in particular the advantage of being explicitly defined and therefore a data-driven procedure can efficiently be deduced through the plug-in principle. In this section, we first define formally our algorithm and then establish its theoretical properties.

\subsection{Data-driven procedure}
\label{sec:datadriven}


Recall the expression of $g^{*,\xi}_{\alpha,p}$ from Theorem~\ref{thr:optimalPredict}. The main purpose is then to build estimators for the regression function $f^*$, as well as to the CDFs $F_{f^*|s}$ and quantile functions $Q^\xi_s$.
We begin by constructing a based estimator $\hat{f}$ of $f^*$ using a collected labeled dataset 
$\mathcal{D}_n=\left\{\left(X_i, S_i, Y_i\right)\right\}_{i=1}^n \stackrel{\text { i.i.d }}{\sim} \mathbb{P}$
of size $n$. Note that, due to tie effects, we consider a jittered version of $\hat{f}$ rather than $\hat{f}$ itself (to enforce Assumption~\ref{ass:density} at the estimation step).
Using this estimate, and in order to construct empirical quantile functions and CDF, we require an unlabeled data set $\mathcal{D}_N=\bigcup_{s\in\mathcal{S}}\mathcal{D}^s$, where
\begin{equation*}
   \mathcal{D}^s= \{(X_i^s,s): i=1,...,N_s \} \stackrel{\text { i.i.d }}{\sim} \mathbb{P}_{X,S|S=s} \enspace ,
\end{equation*}
such that $\sum_{s \in \mathcal{S}} N_s =N$. We define a fixed partition of the index set $\left\{1,..., N_s\right\}$ into two equally sized disjoint subsets $\mathcal{I}_0^s$ and $\mathcal{I}_1^s$, each of cardinality $N_s / 2$, such that $\mathcal{I}_0^s \cup \mathcal{I}_1^s=\left\{1,..., N_s\right\}$. 
For each $j \in\{0,1\}$, we define the corresponding subsample $\mathcal{D}_j^s=\left\{(X_i^s,s) \in \mathcal{D}^s: i \in \mathcal{I}_j^s\right\}$. 
Then, for each $s\in\mathcal{S}$, we use the data in $\mathcal{D}^s_1$ to get the sample $\left\{\hat{f}\left(X_i^s, s\right)+\varepsilon_{i s}\right\}_{i \in \mathcal{I}_0^s}$ and estimate the cumulative CDF of the random variable $\hat{f}(X, S)+\varepsilon_s$ conditioned on $S=s$. Similarly, we the data in $\mathcal{D}^s_2$ to get the sample $\left\{\hat{f}\left(X_i^s, s\right)+\varepsilon_{i s}\right\}_{i \in \mathcal{I}_2^s}$ and estimate the quantile function of $\hat{f}(X, S)+\varepsilon_s$ conditioned on $S=s$.
There the jittering noise, $\varepsilon_{i s} \stackrel{\text{ i.i.d. }}{\sim} U([-\sigma, \sigma])$ are independent from the other data, for some $\sigma>0$ and has to be set by the user --- typically $\sigma = 10^{-6}$ in our experiments. 
\\
Finally, fix some $\xi >0$ and define an estimator $\hat{g}_{\alpha,p}^{\xi}$ of $g^{*,\xi}_{\alpha,p}$
\begin{align}
\label{estimator}
    \hat{g}_{\alpha,p}^{\xi}(x,s)= \hat{Q}_s\circ \hat{F}_{\hat{f}|s}\circ (\hat{f}(x,s)+\varepsilon_s) \enspace,
\end{align} 

where $\varepsilon_s\sim \mathcal{U}[-\sigma,\sigma]$ is independent from all remaining random variables and
\begin{align*}
     \hat{Q}^\xi_s(x)=\left\{\begin{array}{ll}
        \min\{\alpha,\hat{Q}_{\hat{f}|s}(x)\} &\text{if $x\in[0,p]$} , \\
\max\left\{\alpha + \xi, \sum\limits_{s\in\mathcal{S}}\hat{p}_s\hat{Q}_{\hat{f}|s}(x)\right\} &\text{if $x\in(p,1]$},
    \end{array}
    \right.
\end{align*}
with $\hat{p}_s= N_s/N$. In what follows we define $\mathcal{E} = \left\{ \varepsilon_{i,s}  \right\}_{i,s}\cup \left\{\varepsilon_s\right\}$.

\subsection{Theoretical analysis}

The performance of the estimator $\hat{g}^{\xi}_{\alpha,p}$ can be considered from two sides. Obviously, we need to control its prediction error. In addition, to evaluate the unfairness of $\hat{g}^{\xi}_{\alpha,p}$ we also need to build some measure of unfairness that is related to our Definition~\ref{def:PartialDPGeneral}. This is the purpose of the following paragraphs.

\paragraph{DP-tails fairness control.} 
The nature of the fairness constraint makes natural to build a notion of unfairness that measure the difference between CDFs for all $t$ after the threshold $\alpha$. In particular, we define the tail-unfairness of a prediction function $g: \mathbb{R}^d \times \mathcal{S} \rightarrow \mathbb{R}$ as 
\begin{multline}
\label{eq:Unfairness}
\mathcal{U}^{\alpha}(g) = \sup_{s,s'\in\mathcal{S}} \sup _{t \geq\alpha } \left| \mathbb{P}_{X \mid S=s}( {g}(X, S) \leq t) 
\right. 
\\ 
\left. - \mathbb{P}_{X \mid S=s'}\left( {g}(X, S) \leq t \right)\right| \enspace.
\end{multline}
We then can establish the following result.
\begin{theorem}[DP-tails fairness guarantees]
\label{thm:unfairnessControl}
    For any joint distribution $\mathbb{P}$ of $(X, S, Y)$ and any base estimator $\hat{f}$ constructed on labeled data, the estimator $\hat{g}_{\alpha,p}^{\xi}$ defined in Equation~\eqref{estimator} satisfies
\begin{equation}
\label{eq:e-pf}
    \mathbf{E} \left[ \mathcal{U}^{\alpha}(\hat{g}_{\alpha,p}^{\xi}) \right]
    \leq  \dfrac{C} {\sqrt{\min_{s\in\mathcal{S}}\{p_s\}N}}  \enspace,
\end{equation}
where $C>0$ is an absolute constant and the expectation $\mathbf{E} $ is taken over all datasets $\mathcal{D}= \mathcal{D}_n \cup \mathcal{D}_N \cup \mathcal{E}$ used in the construction of $\hat{g}_{\alpha,p}^{\xi}$. 
In addition, for all $s\in\mathcal{S}$, we have
\begin{align}\label{eq:f-pf}
  \mathbf{E}  \left|\mathbb{P}_{X|S=s}\left(\hat{g}_{\alpha,p}^{\xi}(X,S)\leq \alpha\right)-p\right| \leq \frac{C'}{\sqrt{{p_sN}}} \enspace,
\end{align} 
for some absolute constant $C'>0$.
\end{theorem}
One important characteristic of the above result is that it holds without any condition on the data distribution --- it is then a distribution-free guarantee --- neither on the quality of the estimator $\hat{f}$. The theorem says in particular that the fail unfairness decay to zero is only governed by the number of unlabeled data. This result is aligned with the bound on global unfairness obtained in~\cite{chzhen2020fairwb} and its proof can be found in Appendix~\ref{app:proofUnfair}.

\paragraph{Risk control.} Establishing a control on the risk of $\hat{g}^{\xi}_{\alpha,p}$ is more involved and requires additional assumptions, on particular on the efficiency of $\hat{f}$.
\begin{assumption}\label{as1}
  For each $s \in \mathcal{S}$ the univariate measure $\nu_{f * \mid s}$ admits a density $q_s$, which is lower bounded by $\underline{\lambda}_s>0$ and upper-bounded by $\bar{\lambda}_s \geq \underline{\lambda}_s$.  
\end{assumption}
\begin{assumption}\label{as3}
We assume that\\
$i)$ for each $s\in \mathcal{S}$, $x \mapsto f^*(x,s)$ is Lipschitz;\\
$ii)$ the feature $X$ belongs to a compact set;\\
$iii)$ conditional on $S=s$, $\mathbb{P}_{X|S=s}$ admits a density which lower and upper bounded.
\end{assumption}

\begin{assumption}\label{as4}
    There exist positive constants $c$ and $C$ independent from $n, N, N_1, \ldots, N_{|\mathcal{S}|}$, and a positive sequence $b_n: \mathbb{N} \rightarrow \mathbb{R}_{+}$ such that for all $\delta>0$ it holds that
$$
\mathbf{P}\left(\left|f^*(x, s)-\hat{f}(x, s)\right| \geq \delta\right) \leq c \exp \left(-C b_n \delta^2\right)
$$
for almost all $(x, s)$ \emph{w.r.t.} $\mathbb{P}_{X, S}$.
\end{assumption}

The above Assumption provide a rates of convergence for the predictor $\hat{f}$. 
In particular under Assumption~\ref{as3},  this assumption is satisfied by local polynomial estimators for instance~\citep{Audibert_Tsybakov07}.
Besides, note that Assumption~\ref{as3} and~\ref{as4} ensure that
\begin{equation*}
\mathbf{E}\left[\left\|\hat{f}-f^* \right\|_{\infty}\right] \leq C \log(n) b_n^{-1/2},
\end{equation*}
where $\left\| f \right\|_{\infty} := \sup_{(x,s)\in \bbR^d\times \mathcal{S} } | f(x,s)|$ is the sup-norm, for any $f:\bbR^d \times \mathcal{S} \rightarrow \bbR$.

\begin{theorem}
\label{thm:riskBound}
    Let Assumptions \ref{as1}, \ref{as3} and \ref{as4} be satisfied, and set $\sigma \lesssim \min _{s \in \mathcal{S}} N_s^{-1 / 2} \wedge b_n^{-1 / 2}$ and $\xi \lesssim \sqrt{1/N}$, then for all $\alpha \in \mathbb{R}$ and $p\in[0,1]$, the estimator $\hat{g}_{\alpha,p}^{\xi}$ defined in Eq.~\eqref{estimator} satisfies
$$
\mathbf{E}\left|g_{\alpha,p}^*(X, S)-\hat{g}_{\alpha,p}^{\xi}(X, S)\right| \leq C\left(\sqrt{\dfrac{1}{N}}+ \log(n) b_n^{-1/2} \right)
$$
where the leading constant depends only on $\underline{\lambda}_s, \bar{\lambda}_s,\eta$ from Assumptions~\ref{as1}, ~\ref{as3}, and~\ref{as4}.
\end{theorem}
The obtained bound on the risk is similar in spirit to the one obtained in~\cite{chzhen2020fairwb} and is decomposed into two terms. The first one is related to the control of the deviation between a CDF and its empirical counterpart, while the second term relies on controlling the estimation error $\|\hat{f}-f^*\|_{\infty}$ of $\hat{f}$.


\section{Optimal localization for DP-tails fairness}
\label{sec:optimal}
We now turn to our second framework of interest, namely the case where we optimize our procedure with respect to the parameter $p$. This amounts to enforcing the matching of the right tails of the CDF across sensitive groups beyond a given threshold $\alpha$. We then introduce the notion of \emph{relaxed DP-tails}.
%
\begin{definition}[\emph{Relaxed DP-Tails}]
  For each $\alpha\in \mathbb{R}  $, a predictor $g:\mathcal{X}\times \mathcal{S}\to\mathbb{R}$ is said \textrm{$\alpha-$DP-tails fair} if for all $s,s'\in \mathcal{S}$ 
  \begin{align*}
  \label{level}
    \left.\mathbb{P}( g(X,S)\leq t\right|S=s)=\mathbb{P}( g(X,S)\leq t|S=s'),
  \end{align*} 
  for all $t\geq\alpha$. 
We also introduce $\mathcal{F}_\alpha$, the class of all \textrm{$\alpha-$DP-tails fair} regressors. 
  
\end{definition}
Since for every $\alpha$ the set of fair regressors is contained in $\mathcal{F}_\alpha$, the $L_2$ risk of an optimal fair regressor is always greater than or equal to that of the optimal $\alpha-$ DP-tails fair  regressor. Moreover, due to the fact that
\begin{align*}
\inf_{g\in\mathcal{F}_{\alpha}}\mathcal{E}(g) 
 = \inf_{p\in[0,1]}\inf_{g\in\mathcal{F}_{\alpha}^p}\mathcal{E}(g) =\inf_{p\in[0,1]}\mathcal{E}(g^*_{\alpha,p})\enspace , 
\end{align*} 
an $\alpha-$DP-tails fair regressor is obtained as soon as the minimizer $p^*$ of the objective function 
\begin{equation*}
    \mathcal{I}^\alpha(p):= \mathcal{E}(g^*_{\alpha,p})=\mathbb{E}\left[(f^*(X,S)-g^*_{\alpha,p}(X,S))^2\right]\enspace,
\end{equation*} 
is found. 
Once the optimal unfairness proportion $p^*$ has been determined, the optimal \textrm{$\alpha-$DP-tails fair} regressor can be expressed as
$$
g_\alpha^*=g_{\left(\alpha, p^*\right)}^* \enspace ,
$$
whenever $\alpha \leq \sum_{s \in \mathcal{S}} p_s Q_{f^* \mid s}\left(p^*\right)$; otherwise, one can only guarantee the existence of a $\xi$-optimal \textrm{$\alpha-$DP-tails fair} regressor 
$$
g^{\xi,*}_\alpha=g^{\xi,*}_{(\alpha,p^*)}\enspace ,
$$ 
for some small $\xi>0$.
 Thanks to Remark \ref{re:inf}, we can get the formula of $\mathcal{I}^\alpha(p)$
 \begin{multline*}
   \mathcal{I}^\alpha(p)=\sum_{s\in\mathcal{S}}p_s\left[\int_0^p(Q_{f^*|s}(x)-\min\{\alpha, Q_{f^*|s}(x)\})^2dx \right. \\
   \left. +
   \int_p^1\left(Q_{f^*|s}(x)-\max\left\{\alpha,\sum_{s\in\mathcal{S}}p_sQ_{f^*|s}(x)\right\}\right)^2dx\right]\enspace.  
 \end{multline*}
 In particular, under Assumption \ref{ass:density} and by dominated convergence theorem, $\mathcal{I}^\alpha$  is continuous on $[0,1]$ and then the function $p \mapsto \mathcal{I}^\alpha(p)$ admits a minimum  on $[0,1]$.

Nevertheless, due to the lack of differentiability of $\mathcal{I}^\alpha$, it is generally difficult to derive an explicit expression for a minimizer. 

In the statistical setting, we will approximate the optimal unfairness proportion by minimizing \emph{w.r.t.} $p$ an empirical version of $ \mathcal{I}^\alpha(p)$ given by
\begin{equation}
\label{eq:OptimProblemP}
    \widehat{\mathcal{I}}^\alpha(p)=\frac{1}{N}\sum_{s\in\mathcal{S}}\sum_{X^s\in\mathcal{D}_s}\left( \hat{f}(X^s,s)-\hat{g}^0_{\alpha,p}(X^s,s)\right)^2\enspace,
\end{equation}
where $\hat{g}^0_{\alpha,p}$  is given by setting $\xi=0$ in Eq~\eqref{estimator}.
Hopefully, solving $\min_{p\in [0,1]} \widehat{\mathcal{I}}^\alpha(p)$ and be performed efficiently since $ \hat{g}^0_{\alpha,p}$ is a piecewise constant function \emph{w.r.t.} $p$. Once the minimizer $\hat{p}$ of $\hat{\mathcal{I}}$ is obtained, we construct $\hat{g}_{\alpha}^{\xi}:=\hat{g}_{\alpha, \hat{p}}^{\xi}$ as an approximation to the optimal $\alpha-$DP-tails predictor. 
Section~\ref{sec:numeric} provides a numerical study for $\hat{g}_{\alpha}^{\xi}$ that shows that it performs well both on synthetic and real datasets.



\section{Numerical study}
\label{sec:numeric}
This section is dedicated to the numerical evaluation of our proposed algorithms, $\hat{g}_{\alpha,p}^{\xi}$ and $\hat{g}_\alpha^{\xi}$, that are devoted to address the tail-unfairness issue respectively with a fixed and optimized proportion $p$. We recall that $\hat{g}_\alpha^{\xi}$ is obtained from $\hat{g}_{\alpha,p}^{\xi}$ by optimizing over the parameter $p$; for this step, we use the bounded Brent’s method~\cite{Brent1973}.
In all our experiments, we set $\xi = 10^{-5}$; this value does not affect the results.
We begin with a simulation study and then consider applications to classical real-world datasets commonly used in the fairness literature.

\paragraph{Evaluation metrics.} All our experiments are evaluated using the empirical counterparts of the risk $R(g)$ ---see Eq.~\eqref{eq:risk}--- namely the mean squared error $\mathrm{MSE}:=\frac{1}{M}\sum_{(X,S,Y) \in \mathcal{D}_M} (Y-\hat f(X,S))^2$, and of the tail-unfairness level $\mathcal{U}^{\alpha}(g)$ for a threshold $\alpha\in\mathbb{R}$ ---see Eq.~\eqref{eq:Unfairness}--- of the considered estimator. Besides, for $\alpha \to -\infty$, $\mathcal{U}^{\alpha}(g)$ reduces to global unfairness, namely the Kolmogorov–Smirnov statistic $\mathrm{KS}:=\max_{s,s'}\sup_{t\in\R}|\hat{F}_{\hat f\mid S=s}(t)-\hat{F}_{\hat f\mid S=s'}(t)|$.

\paragraph{Baselines\footnote{The code (in python) and data required to reproduce all numerical experiments are provided in the supplementary material.}.}
We consider:
(i) \texttt{Unfair} the base regressor $\hat f$ that is the specific case where we do not enforce any constraint, that is, when $\alpha \rightarrow +\infty$;
(ii)~\texttt{OT}~\cite{chzhen2020fairwb}: the case where we enforce complete fairness, corresponding to the case where $\alpha \rightarrow -\infty$.

\paragraph{Estimation protocol.} We build three datasets: a labeled training set, an unlabeled set for the calibration of fairness with the right levels, and a test set to evaluate risk and unfairness:\\
i) the \emph{labeled training set} $\mathcal{D}_n$ of size $n$ is used to learn a baseline regression function $\hat{f}$;
\\
ii) the \emph{unlabeled calibration dataset} $\mathcal{D}_N$ of size of $N$ is generated in the same way but without any target variable $y$. It serves as a pool of samples for the DP-tails fairness adjustment stage, allowing the model to estimate and correct group-level disparities;
\\
iii) finally, the test dataset $\mathcal{D}_M$ of size of $M$ is also drawn from the same distribution and used exclusively to evaluate the performance and DP-tails fairness of the learned predictors.

\subsection{Simulation study}

\begin{figure}[t]
  \centering
  \subfloat[Unfair]{\includegraphics[width=0.5\linewidth]{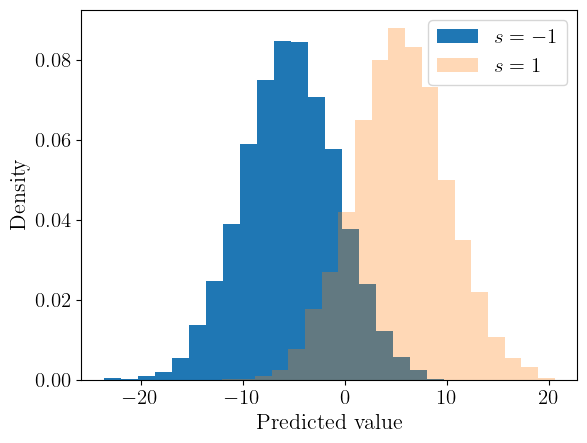}}\hfill
  \subfloat[DP-tails fair]{\includegraphics[width=0.5\linewidth]{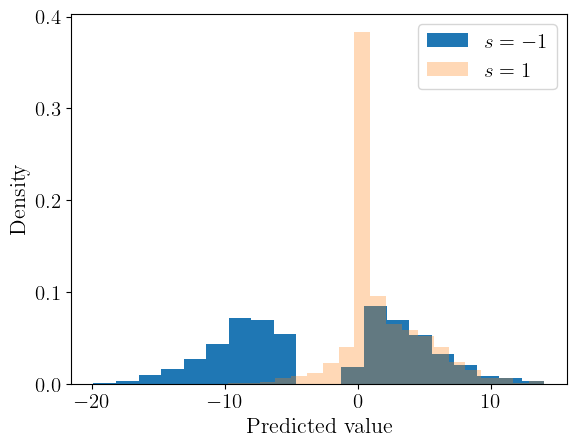}}\hfill
  \caption{Histogram of predictions on synthetic data before (left) and after (right) enforcing $(\alpha,p)$-DP-tails fairness. The constraint parameters are set (arbitrary) to $\alpha=0$, $p=0.5$}
  \label{fig_hist_test}
\end{figure}

The scenario that we consider is rather simple but perfectly illustrates the action of our constraint.
The considered model is 
$$
Y=3\left(x_1+x_2+x_3+S\right)+\varepsilon, \quad \varepsilon \sim \mathcal{N}(0,1)\enspace ,
$$
where $X=\left(x_1, x_2, x_3\right)\in\mathbb{R}^3$ is independent of $ \varepsilon$ and is drawn from a standard normal distribution in $\mathbb{R}^3$.
The sensitive attribute $S \in\{0,1\}$ is determined by the sign of the first feature, that is, $S=0$ if $x_1<0$ and $S=1$ otherwise --- the above expression means that $f^*(X,S)= 3\left(x_1+x_2+x_3+S\right)$. 
Generating the outputs in this way introduces an explicit dependence between $Y$ and the sensitive attribute $S$, creating an intentionally unfair scenario.

Except specified, in this set of experiments, we considered $N=10,000$ for the calibration and $M = 10,000$ for evaluating the metrics.

%

\paragraph{Visualization of distributions.}
Our first experiment aims at showing the impact of our constraint on the distributions $\nu_{f^*|s}$. Since our goal is to understand how fairness is enforced, we assume in this experience that $f^*$ is given and we only deal with the part of the study relying on the calibration of fairness -- this means that we use the estimator $\hat{g}^{\xi}_{\alpha,p}$ but we replace $\hat{f}$ by $f^*$ everywhere it is used. 
In Figure~\ref{fig_hist_test}, we draw the distributions of $f^*(X,s)$ for $s\in \{0,1\}$ before and after enforcing tail-fairness. 
While the left plot illustrates a situation of complete unfairness, the right plot shows that enforcing DP-tails fairness leads to two distinct behaviors. One distribution (blue) exhibits a gap in its support before the threshold $\alpha$, whereas the other (orange) displays an accumulation of mass just after $\alpha$. These observations are confirmed by Figure~\ref{fig_cdf_test}, which in particular highlights the perfect agreement between the two CDFs beyond the threshold $\alpha$, both in the case of $(\alpha,p)$-DP-tails fairness (middle plot) and $\alpha$-DP-tails fairness (right plot).

%
%

\begin{figure}[t]
  \centering
  \subfloat[Unfair]{\includegraphics[width=0.32\linewidth]{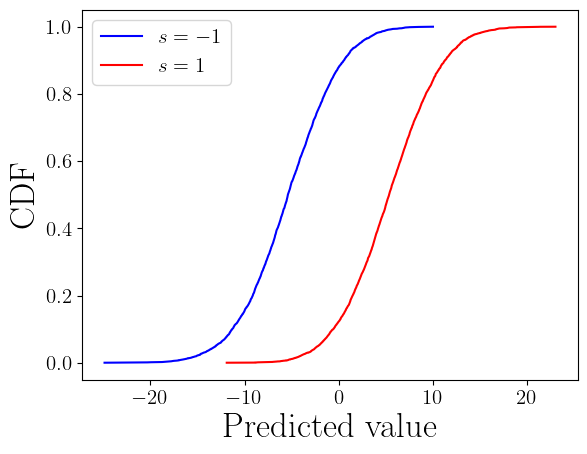}}\hfill
  \subfloat[$(\alpha,p)-$DP-tails]{\includegraphics[width=0.32\linewidth]{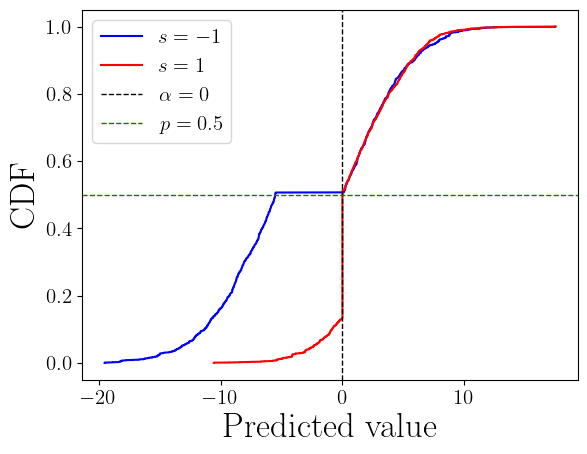}}\hfill
  \subfloat[$\alpha-$DP-tails]{\includegraphics[width=0.32\linewidth]{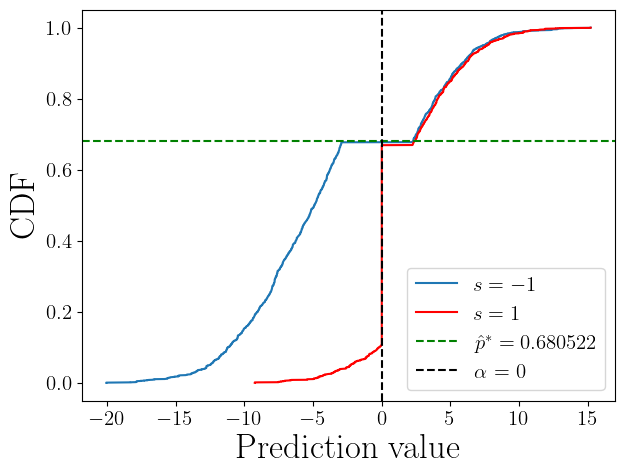}}\hfill
  \caption{Empirical CDF of the predictions on the synthetic data before and after enforcing DP-tails fairness. Plot (b): the parameters are set (arbitrary) to $\alpha=0$, $p=0.5$; plot (c): $\alpha=0$ and $p$ minimizes~\eqref{eq:OptimProblemP}.}
  \label{fig_cdf_test}
\end{figure}

Before proceeding to further investigations, we illustrate Theorem~\ref{thm:unfairnessControl} and examine the decay of DP-tails unfairness toward $0$ as the size $N$ of the unlabeled dataset increases. Figure~\ref{fig:placeholder} displays this evolution, and we observe that a moderate amount of unlabeled data is already sufficient to achieve a good calibration of tail unfairness.
This behavior holds not only for $\hat{g}_{\alpha,p}^{\xi}$ but also for $\hat{g}_{\alpha}^{\xi}$ for any value of $\alpha$. Therefore, in what follows, we mainly focus on the risk (MSE) and the global unfairness (KS) as $\alpha$ varies.

\begin{figure}[t]
    \centering
    \includegraphics[width=0.5\linewidth]{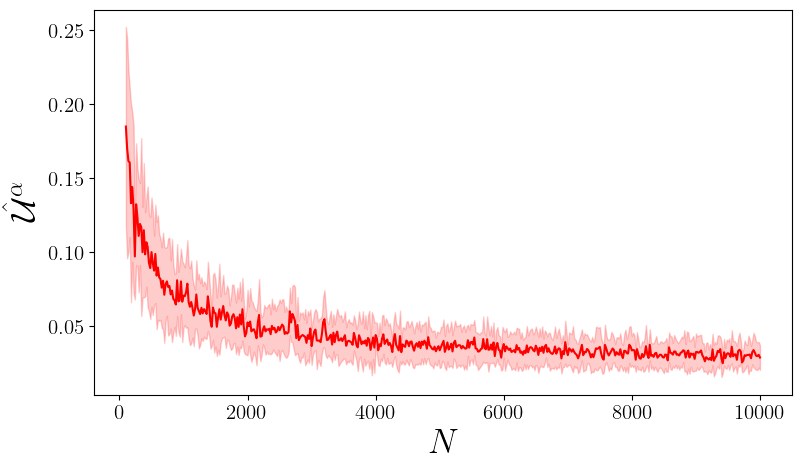}
    \caption{Evolution of the tail unfairness of $\hat{g}_{\alpha,p}^{\xi}$ with respect to the size $N$ of the unlabeled dataset on the synthetic data.}
    \label{fig:placeholder}
\end{figure}

\paragraph{Evolution with respect to $\alpha$.}
Our second study investigates the behavior of $\hat{g}^{\xi}_{\alpha}$ as $\alpha$ varies. 
Figure~\ref{fig:MSESimu} highlights the expected behavior: stricter constraints lead to a larger MSE and smaller KS for the prediction function $\hat{g}_{\alpha}^{\xi}$. This illustrates in particular the advantage of DP-tails fairness constraints, which allow performance closer to that of the base predictor, especially when fairness is required only over a small region of the label space.

\begin{figure}[t]
    \centering
    \includegraphics[width=0.5\linewidth]{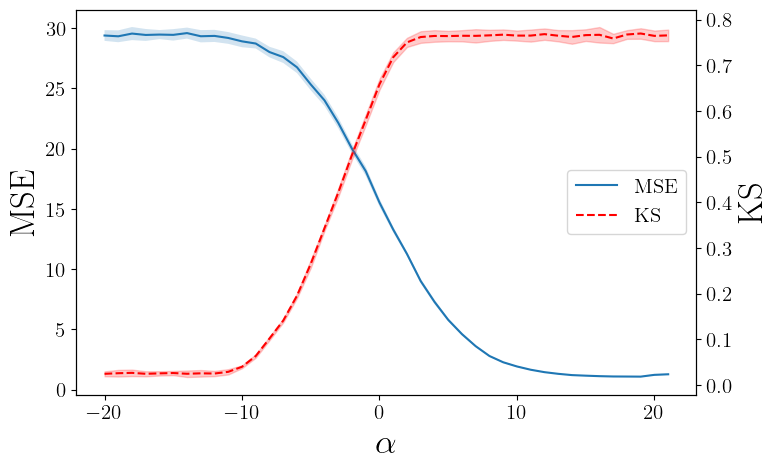}
    \caption{Evolution of the MSE and KS distance  of $\hat{g}_{\alpha}^{\xi}$ with respect to $\alpha$ on the synthetic data.}
    \label{fig:MSESimu}
\end{figure}

\subsection{Real data}

We now apply our methodology to three real datasets commonly used in the context of fairness.
\\
$\bullet$ \texttt{Law School} dataset drawn from the LSAC National Longitudinal Bar Passage Study \cite{wightman1998lsac}, comprises 22,342 samples. We consider a regression task that predicts students' GPA (scaled to $[0,1]$ ) with race as the protected attribute (white---20,641 with $s=1$ vs. non-white---1,701 with $s=0$). 
\\
$\bullet$ \texttt{Communities\&Crime (CRIME)} deals with socio-economic and demographic data of U.S \cite{redmond2002data} communities collected from the 1990 U.S. Census and law enforcement records, with 1,994 instances. The task is to predict the violent crime rate per population. We consider race-related attributes, in particular the proportion of African-American residents, as sensitive attributes, which obtains 1,032 instances for 
$s=0$ and 962 instances for 
$s=1$.
\\
$\bullet$ \texttt{California Housing} dataset contains 20,640 census block groups from the 1990 U.S. Census \cite{geron2017hands-on}. The task is to predict the median house value, which takes values in $[0.15, 5]$. We define a binary sensitive attribute based on latitude, splitting the data into northern region ($s=0$ with 10,313 groups) and southern region ($s=1$ with 10,327 groups) of California.

For each dataset, we split the sample by holding out $20\%$ as a test set. From the remaining $80\%$, we use $70\%$ to train $\hat{f}$ with a \emph{Random Forest} (with 200 trees) and the remaining $30\%$ to estimate the quantile functions and conditional CDFs (and eventually the optimal proportion $\hat{p}$). We repeat this procedure 20 times and report the considered metrics evaluated on the test set along with their standard deviations.

\begin{figure}[t]
  \centering
  \subfloat[Unfair]{\includegraphics[width=0.33\linewidth]{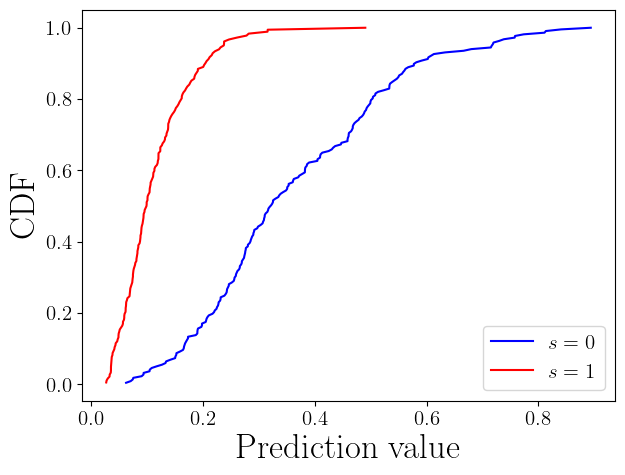}}\hfill
  \subfloat[$(\alpha,p)-$DP-tails]{\includegraphics[width=0.33\linewidth]{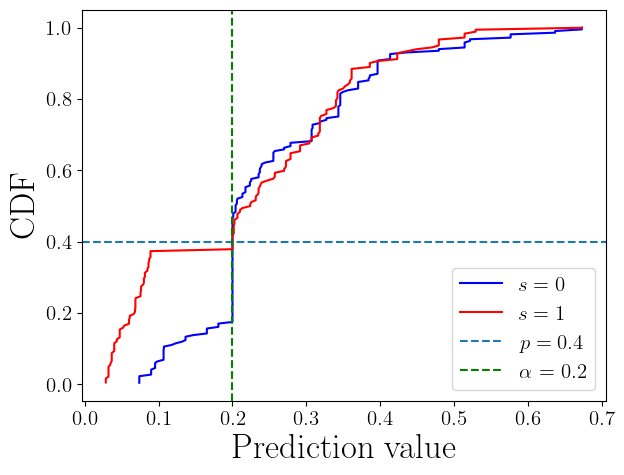}}\hfill
  \subfloat[$\alpha-$DP-tails]{\includegraphics[width=0.33\linewidth]{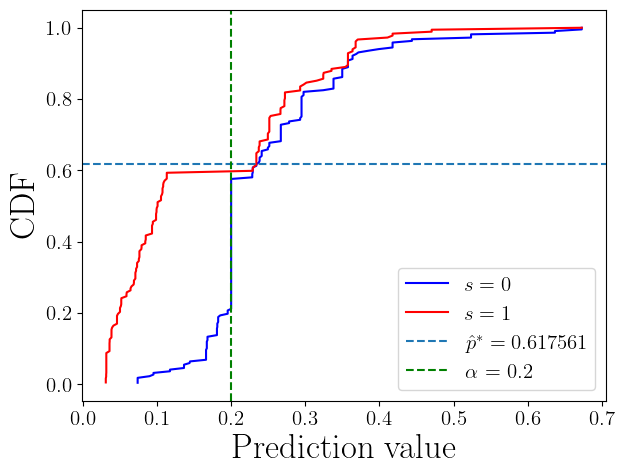}}\hfill
    \caption{Empirical CDF of the predictions on the \texttt{Crime} before and after enforcing DP-tails fairness. Plot (b): the parameters are set (arbitrary) to $\alpha=0.2$, $p=0.4$; Plot (c): $\alpha=0.2$ and $p$ minimizes~\eqref{eq:OptimProblemP}. More runs can be seen in Fig.~\ref{fig:sumcrime} in Appendix~\ref{app:additional numeric}.}
\label{fig:three-in-row}
\end{figure}

Similarly to the simulation study, DP-tails fairness is effectively achieved in all settings. For instance, Figure~\ref{fig:three-in-row} illustrates this performance on the \texttt{CRIME} dataset (see Appendix~\ref{app:additional numeric} for the other datasets). We therefore focus on the analysis of risk (MSE) and global unfairness (KS) of the considered estimation methods. As shown in Figure~\ref{fig:realDataMSEvsKS}, the evolution of risk and global unfairness mirrors what was observed in the simulation study.
Overall, this analysis highlights the benefit of our methods, which provide a continuum of solutions between the \texttt{unfair} predictor $\hat{f}$ and the \texttt{OT} solution that achieves global fairness. In particular, our approach effectively localizes fairness to the tails of the prediction distribution.

\begin{figure}[t]
  \centering
\subfloat[\texttt{Law School}]{\includegraphics[width=0.5\linewidth]{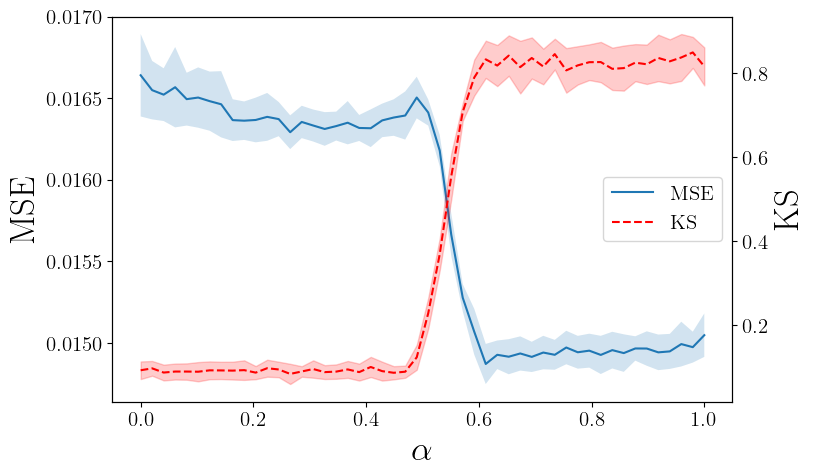}}\hfill
\subfloat[\texttt{CRIME}]{\includegraphics[width=0.5\linewidth]{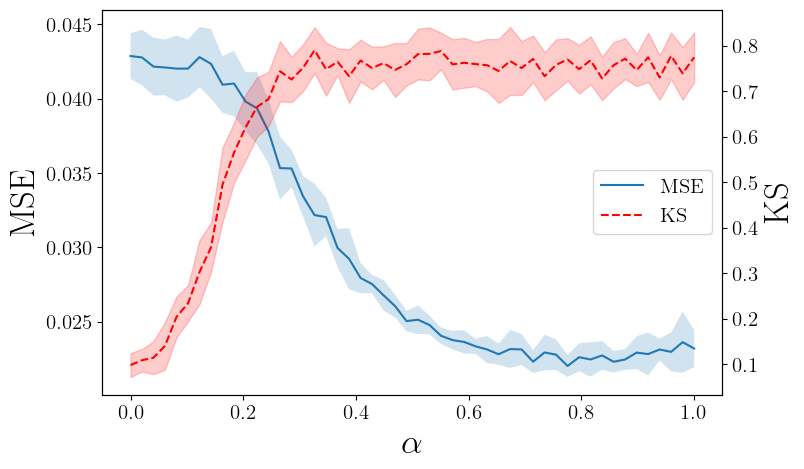}}\\ \centering
\subfloat[\texttt{California Housing}]{\includegraphics[width=0.5\linewidth]{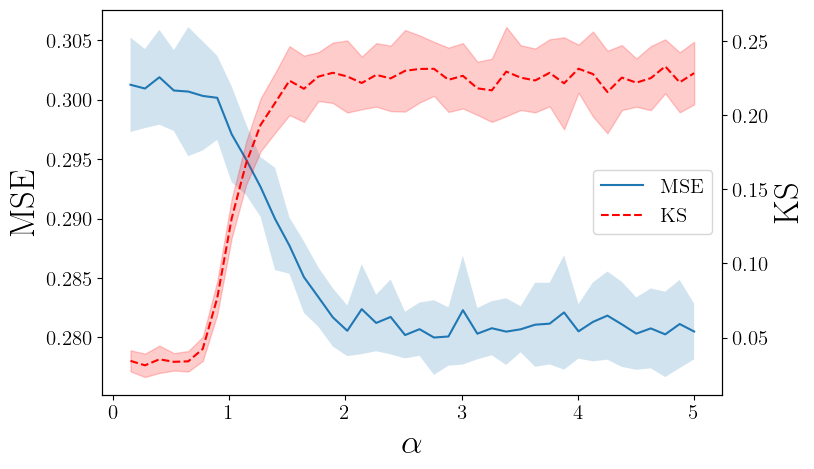}}
\caption{Evolution of the MSE and KS distance of $\hat{g}_{\alpha}^{\xi}$ with respect to $\alpha$ on the real data.}
\label{fig:realDataMSEvsKS}
\end{figure}

\section{Conclusion}
\label{sec:conclusion}
In this work, we formulate the problem of regression under Demographic Parity in a setting where fairness is enforced only for predictions above a given threshold. We consider two specific scenarios: one in which the proportion of the prediction distribution exceeding the threshold is fixed in advance, and another in which this proportion is optimized so as to minimize the risk. We address both problems using optimal transport tools, and demonstrate the effectiveness of our approach numerically, while also providing theoretical guarantees.

From a technical perspective, our framework could be extended to enforce fairness over an interval of prediction values, rather than only on the tails of the distribution. Handling unions of intervals, however, appears more challenging and constitutes an interesting direction for future research. Another complementary line of work would be to investigate DP-tails fairness under approximate distribution matching, \emph{i.e.,} relaxing exact tails-demographic parity to approximate fairness constraints.

\clearpage

\section*{Impact Statement}
This work proposes statistical methods for enforcing tail-based fairness constraints in regression. Such methods may contribute to reducing disparate treatment in high-risk decision-making settings where only tail predictions require/trigger interventions, for instance in clinical risk stratification or allocation of limited public resources. However, our fairness constraints rely on the appropriate choice of sensitive attributes $S$ and a good calibration of the proportion and threshold (that is the parameter $(\alpha,p)$) and mis-specification may lead to unintended allocation outcomes or mask deeper structural inequities. Therefore, any deployment of the proposed methods should be accompanied by domain-specific validation and expert oversight.

\bibliography{biblio}
\bibliographystyle{icml2026}

\newpage
\appendix
\setcounter{equation}{0}
\onecolumn

\begin{center}
\bf{\Large Supplementary Materials}
\end{center}

\vspace*{0.25cm}

\paragraph{Appendix overview.} Appendix~\ref{app:OptimalRules} is dedicated to the proofs related to the optimal solutions and Appendix~\ref{app:proofEmpirical} gathers the proofs for our empirical guarantees. In addition, Appendix~\ref{app:additional numeric} complements our numerical analysis. 

\section{Proofs of optimal rules} 
\label{app:OptimalRules}
This appendix provides the proofs relying on the optimal DP-tails fair predictor. In particular we prove here Theorems~\ref{BaryCenter} and~\ref{thr:optimalPredict}.

\subsection{Proof of Theorem \ref{BaryCenter} }

The idea for proving Theorem~\ref{BaryCenter} is based on the results established in \cite{chzhen2020fairwb,gouic2020projection}.
\begin{lemma}
\label{le:1}
For all $\alpha\in\R$ and $p\in[0,1]$, we have 
\begin{align*}
    \inf_{g\in \mathcal{F}^p_\alpha}\mathbb{E}\left[(f^*(X,S)-g(X,S))^2\right]=\inf_{g\in \mathcal{F}^p_\alpha}\sum_{s\in \mathcal{S}}p_s\mathcal{W}_2^2(\nu_{f^*|s},\nu_{g|s}).
\end{align*}
\end{lemma}
\begin{proof}
     Firstly, we prove that for all $g\in \mathcal{F}^p_\alpha$ $$\mathbb{E}\left[(f^*(X,S)-g(X,S))^2\right]\geq \sum_{s\in \mathcal{S}}p_s\mathcal{W}_2^2(\nu_{f^*|s},\nu_{g|s}) .$$ Indeed, for each $s\in \mathcal{S}$, we have $$\mathbb{E}[(f^*(X,S)-g(X,S))^2|S=s]\geq \mathcal{W}_2^2(\nu_{f^*|s},\nu_{g|s}),$$ this implies that $$\mathbb{E}\left[(f^*(X,S)-g(X,S))^2\right]=\sum_{s\in\mathcal{S}}p_s\mathbb{E}[(f^*(X,S)-g(X,S))^2|S=s]\geq \sum_{s\in\mathcal{S}}p_s\mathcal{W}_2^2(\nu_{f^*|s},\nu_{g|s}).$$ Therefore, taking the minimum over $\mathcal{F}^p_\alpha$ proves $$ \inf_{g\in \mathcal{F}^p_\alpha}\mathbb{E}\left[(f^*(X,S)-g(X,S))^2\right]\geq \inf_{g\in \mathcal{F}^p_\alpha}\sum_{s\in \mathcal{S}}p_s\mathcal{W}_2^2(\nu_{f^*|s},\nu_{g|s}). $$ In addition, since for each $s\in \mathcal{S}$ the distribution $\nu_{f^*|s}$ admits a density, there exists a optimal transport map $T_s: \mathbb{R}\to \mathbb{R}$ such that $T_s(f^*(X,S))|S=s\sim \nu_{g|s}$ and 
   $$\mathcal{W}_2^2(\nu_{f^*|s},\nu_{g|s})=\mathbb{E}[(f^*(X,S)-T_s(f^*(X,S)))^2|S=s].$$ Denote $h(x,s)=T_s(f^*(x,s))$, we have  $\left(\nu_{h|s}\right)_{s\in \mathcal{S}}=\left(\nu_{g|s}\right)_{s\in \mathcal{S}}$, which implies that $h\in\mathcal{F}^p_\alpha$. Thus,
   \begin{align*}
\sum_{s\in \mathcal{S}}p_s\mathcal{W}_2^2(\nu_{f^*|s},\nu_{g|s})&=\sum_{s\in \mathcal{S}}p_s\mathbb{E}[(f^*(X,S)-h(X,S))^2|S=s]\\&=\mathbb{E}[(f^*(X,S)-h(X,S))^2]\\&\geq \inf_{g\in \mathcal{F}^p_\alpha}\mathbb{E}\left[(f^*(X,S)-g(X,S))^2\right],
   \end{align*}
   which implies $$\inf_{g\in\mathcal{F}^p_\alpha}\sum_{s\in \mathcal{S}}p_s\mathcal{W}_2^2(\nu_{f^*|s},\nu_{g|s})\geq \inf_{g\in \mathcal{F}^p_\alpha}\mathbb{E}\left[(f^*(X,S)-g(X,S))^2\right].  $$
\end{proof}
The following lemma aims to make the set of measures $\{\left(\nu_{g|s}\right)_{s\in\mathcal{S}}|g\in\mathcal{F}^p_\alpha\}$ more tractable.
\begin{lemma}
\label{le:2} For all $\alpha\in \mathbb{R}$ and $p\in[0,1]$,
    $$ \mathcal{M}^p_\alpha=\{\left(\nu_{g|s}\right)_{s\in\mathcal{S}}|g\in\mathcal{F}^p_\alpha\}.$$
\end{lemma}

\begin{proof}
  Let $g\in \mathcal{F}^p_\alpha$ and $s\in\mathcal{S}$, for all $A\in \mathcal{B}(\mathbb{R})$
  \begin{align*}
      \nu_{g|s}(A)=\underbrace{\nu_{g|s}(A\cap(-\infty,\alpha])}_{\nu^{-}_{g|s}(A)}+\underbrace{\nu_{g|s}(A\cap(\alpha,+\infty))}_{\nu^{+}_{g|s}(A)}.
  \end{align*}
  Noticing that $\nu_{g|s}^{+}$ would be the same for all $s$, thus $\nu_{g|s}^{+}=:\nu_{g}^{+}$. Since $\nu^-_{g|s}(\R)=\nu^-_{g|s}((-\infty,\alpha])=F_{g|s}(\alpha)=p$, it is easy to see that $$\{\left(\nu_{g|s}\right)_{s\in\mathcal{S}}|g\in\mathcal{F}^p_\alpha\}\subset \mathcal{M}^p_\alpha.$$ Conversely, let $\left(\nu^++\nu^-_s\right)_{s\in\mathcal{S}} \in \mathcal{M}^p_\alpha$. For all $(x,s)\in\mathcal{X}\times \mathcal{S}$, denote $$g(x,s)=Q_{\nu^++\nu_s^-}\circ F_{\nu_{f^*|s}}\circ f^*(x,s).$$ Since for each $s \in \mathcal{S}$, $F_{f^*|s}$
is continuous and we have that $F_{f^*|s}(f^*(X,s))$
is distributed according to a Uniform distribution on $[0,1]$. 
 And then the cumulative distribution of $g(X,S)|S=s$ can be expressed as
$$\begin{aligned}
\mathbb{P}_{X|S=s}(g(X, S) \leq t) & =\mathbb{P}_{X \mid S=s}\left(Q_{\nu^++\nu^-_s} \circ F_{f^* \mid s} \circ f^*(X, s) \leq t\right)\\&= \mathbb{P}_{X \mid S=s}\left(F_{f^*\mid s}(f^*(X, s)) \leq F_{\nu^++\nu^-_s}(t)\right)\\
&=F_{\nu^++\nu^-_s}(t).
\end{aligned}
$$   Additionally, for all $t\geq \alpha$ and $s\in\mathcal{S}$,
\begin{align*}
     \mathbb{P}( g(X,S)\leq t|S=s) &= (\nu^++\nu_s^-)((-\infty,t])\\
     &=p+\nu^+((\alpha,t])\enspace ,
\end{align*}
and this implies that 
$$
\mathbb{P}(\alpha\leq g(X,S)\leq t|S=s)=\mathbb{P}(\alpha\leq g(X,S)\leq t|S=s')\quad \text{for all }s,s'\in \mathcal{S} \enspace.
$$ 
Clearly we have that $F_{g|s}(\alpha)=p$ for all $s\in\mathcal{S}$. Hence $g\in\mathcal{F}_\alpha^p$ and thus 
$$
\mathcal{M}^p_\alpha\subset\{\left(\nu_{g|s}\right)_{s\in\mathcal{S}}|g\in\mathcal{F}^p_\alpha\}\enspace.
$$
\end{proof} 
\begin{proof}[Proof of Theorem \ref{BaryCenter}]
This result is directly deduced from Lemma \ref{le:1} and Lemma \ref{le:2}.   
\end{proof}

\subsection{Proof of Theorem \ref{thr:optimalPredict}}
Theorem~\ref{thr:optimalPredict} will be proven by using the quantile form of 2-Wasserstein distance, given by
\begin{eqnarray*}
 \mathcal{W}_2^2(\mu, \nu) = \int_0^1\left|Q_\mu(t)-Q_\nu(t)\right|^2 \mathrm{~d} t  \enspace .
\end{eqnarray*}
%
Thanks to Theorem~\ref{BaryCenter} and the quantile form of Wasserstein distance, we have 
\begin{eqnarray}
\label{eq:wqu}
   \inf_{g\in \mathcal{F}^p_\alpha}\mathbb{E}\left[(f^*(X,S)-g(X,S))^2\right]&=&\inf_{\left(\nu^++\nu^-_s\right)_{s\in\mathcal{S}}\in \mathcal{M}_\alpha}\sum_{s\in \mathcal{S}}p_s\mathcal{W}_2^2(\nu_{f^*|s},\nu^++\nu^-_s)
 \nonumber
   \\
&= &\inf_{\left(\nu^++\nu^-_s\right)_{s\in\mathcal{S}}\in\mathcal{M}^p_{\alpha}}\sum_{s\in\mathcal{S}}p_s\int_0^1\left( Q_{f^*|s}(x)-Q_{\nu_++\nu_s^-}(x) \right)^2 dx.
\end{eqnarray}

Denote by $\mathcal{Q}^p_{\alpha}$ the set of all families $\left(Q_s\right)_{s\in\mathcal{S}}=\left(Q^{\mathbf{-}}_s\mathds{1}_{[0,p]}+Q^+\mathds{1}_{(p,1]}\right)_{s\in \mathcal{S}}$ such that
\begin{enumerate}
    \item[i.] For all $s\in\mathcal{S}$, $Q^{\mathbf{-}}_s\mathds{1}_{[0,p]}+Q^\mathbf{+}\mathds{1}_{(p,1]}:[0,1]\to\mathbb{R}\cup\{-\infty,+\infty\}$ is non-decreasing and left continuous on $(0,1)$;
\item[ii.] for all $s\in\mathcal{S}$, $Q^{\mathbf{-}}_s(x)\leq \alpha$ for all $x\in[0,p]$;
\item[iii.] $Q^\mathbf{+}(x)>\alpha$ for all $x\in(p,1]$.
\end{enumerate}
\begin{lemma}\label{le:quantile set}
     For all $\alpha\in\mathbb{R}$ and $p\in[0,1]$, we have  $$\mathcal{Q}_\alpha^p=\{\left(Q_{\nu_s}\right)_{s\in\mathcal{S}}:\left(\nu_s\right)_{s\in\mathcal{S}}\in \mathcal{M}_\alpha^p\}.$$ 
\end{lemma}
\begin{proof}
    Let $\left(\nu^++\nu^-_s\right)_{s\in\mathcal{S}}\in\mathcal{M}^p_{\alpha}$. For each $s\in\mathcal{S},$ we have 
    \begin{align*}
        Q_{\nu^++\nu_s^-}(x)&=\begin{cases}
            \inf\{t|\nu_s^-((-\infty,t])\geq x \}=:Q^{\mathbf{-}}_s(x),\quad x\in [0,p]\\
             \inf\{t|p+\nu^+((\alpha,t])\geq x \}=:Q^\mathbf{+}(x), \quad x\in (p,1]
        \end{cases} \\&=Q^{\mathbf{-}}_s(x)\mathds{1}_{[0,p]}(x)+Q^\mathbf{+}(x)\mathds{1}_{(p,1]}(x).
    \end{align*} 
    Since $\nu_s^-((-\infty,\alpha])=p$, we have $Q^{\mathbf{-}}_s(x)\leq \alpha$ for all $x\in[0,p]$. In addition,   $Q^\mathbf{+}(x)\geq\alpha$ for all $x\in(p,1]$. Assume that there exists $x_0\in(p,1]$ such that $Q^+(x_0)=\alpha$. Then there exists, $\{t_n\}_{n\in\mathbb{N}}\subset (\alpha,+\infty)$ and $t_n \downarrow \alpha$ such that $\nu^+((\alpha,t_n])\geq x_0-p$ for all $n$. Since $\nu^+$ is finite measure, we can get that $$0=\nu^+(\varnothing)=\lim_{n\to+\infty}\nu^+((\alpha,t_n])\geq x_0-p>0,$$ which is a contradiction. Thus, we obtain that $Q^+(x)>\alpha$ for all $x\in(p,1]$ and hence $\left(Q_{\nu^++\nu_s^-}\right)_{s\in\mathcal{S}}\in\mathcal{Q}_{\alpha}^p$.\\

    Conversely, let $\left(Q^{\mathbf{-}}_s\mathds{1}_{[0,p]}+Q^\mathbf{+}\mathds{1}_{(p,1]}\right)_{s\in\mathcal{S}}\in\mathcal{Q}_\alpha^p$. Let $X\sim\mathcal{U}([0,1])$ and $$Y_s=Q^{\mathbf{-}}_s(X)\mathds{1}_{[0,p]}(X)+Q^\mathbf{+}(X)\mathds{1}_{(p,1]}(X).$$ Denote $Q_s=Q^{\mathbf{-}}_s\mathds{1}_{[0,p]}+Q^\mathbf{+}\mathds{1}_{(p,1]}$ and $F_s$ is the CDF of $Y_s$, then for all $x\in[0,1]$ we have \begin{align}\label{eq:uniform}
       Q_s(x)=\inf\{y\in\R: F_s(y)\geq x\},
    \end{align} which means that  $Q_s$ is quantile function of distribution of $Y_s$. Indeed, let us define
    $$Q_s^{-1}(x)=\sup\{t\in(0,1):Q_s(t)\leq x\},\quad x\in\R.$$
  Thanks to  \citep[Lemma 1]{wacker2023pleasetextitanothernotegeneralized}, for all $x\in [0,1]$ and $y\in\R$ $$ Q_s(x)\leq y\Leftrightarrow x\leq Q^{-1}_s(y).$$ Then for all $t\in\mathbb{R}$,
    $$F_s(t)=\mathbb{P}(Y_s\leq t)=\mathbb{P}(Q_s(X)\leq t)=\mathbb{P}(X\leq Q_s^{-1}(t))=Q_s^{-1}(t).$$
Thus, for all $x\in[0,1]$ we obtain that \begin{align*}
   \inf\{y\in\R:F_s(y)\geq x\}=\inf\{y\in\R:Q_s^{-1}(y)\geq x\}=\inf\{y\in\R:y\geq Q_s(x)\}=Q_s(x).
 \end{align*}

   Denote by $\nu_{Y_s}$ the distribution of $Y_s$. It can be seen that $Q_{s}(p)\leq \alpha$ and hence $$\nu_{Y_s}((-\infty,\alpha])=F_s(\alpha)\geq p.$$ In addition, since $Q_s(p+\varepsilon)>\alpha$ for all $\varepsilon>0$,
    $$ F_s(\alpha)<p+\varepsilon,$$ and let $\varepsilon\to 0$, to get that $F_s(\alpha)\leq p.$ Thus $F_s(\alpha)=p$ for all $s\in\mathcal{S}$. 

    For all $t>\alpha$ and $s\in\mathcal{S}$, since $Q_s^-(x)\leq \alpha$ for all $x\leq p$, we have
    \begin{align*}
       \nu_{Y_s}((\alpha,t])= \mathbb{P}(\alpha< Y_s\leq t)&=\mathbb{P}(\alpha<Q^{\mathbf{-}}_s(X)\mathds{1}_{(0,p]}(X)+Q^\mathbf{+}(X)\mathds{1}_{(p,1]}(X)\leq t)\\
        &=\mathbb{P}(X\in(p,1],\alpha<Q^\mathbf{+}(X)\mathds{1}_{(p,1]}(X)\leq t).
    \end{align*}
    This holds independently from the value of  $s$, thus $(\nu_{Y_s})_{s\in\mathcal{S}}\in\mathcal{M}_\alpha^p$. Therefore,
    $$\mathcal{Q}_\alpha^p=\{\left(Q_{\nu_s}\right)_{s\in\mathcal{S}}:(\nu_s)_{s\in\mathcal{S}}\in \mathcal{M}_s^p\}.$$
\end{proof} 
From this Lemma and eq.(\ref{eq:wqu}), we can get that
\begin{equation}
\label{eq:main}
    \inf_{g\in \mathcal{F}^p_\alpha}\mathbb{E}\left[(f^*(X,S)-g(X,S))^2\right]=\inf_{\left(Q_s\right)_{s\in\mathcal{S}}\in\mathcal{Q}_\alpha^p}  \sum_{s\in\mathcal{S}}p_s\int_0^1\left( Q_{f^*|s}(x)-Q_s(x) \right)^2 dx.
\end{equation} 
We will denote by $(\mathcal{P'})$ the problem in the right hend side of eq.(\ref{eq:main}) and by $\mathcal{F}[0,1]$ the set of all functions in $[0,1]$. Define $\mathcal{L}: (\mathcal{F}[0,1])^{|\mathcal{S}|}\to \mathbb{R}$, such that $$\mathcal{L}(Q)=\sum_{s\in\mathcal{S}}p_s\int_0^1\left( Q_{f^*|s}(x)-Q_s(x) \right)^2 dx,$$ for all $Q=(Q_s)_{s\in\mathcal{S}}\in (\mathcal{F}[0,1])^{|\mathcal{S}|}$.
The next Lemma give us an approximation of optimal solution of $(\mathcal{P}')$.

\begin{lemma}
\label{le:sol}
Define $Q^*=(Q^*_s)_{s\in\mathcal{S}}\in(\mathcal{F}[0,1])^{|\mathcal{S}|}$ such that for each $s\in\mathcal{S}$
    $$ 
    Q^*_s(x)=
        \min\{\alpha, Q_{f^*|s}(x)\}\mathds{1}_{[0,p]} (x) +
        \max\{\alpha,\sum_{s\in\mathcal{S}}p_sQ_{f^*|s}(x)\}\mathds{1}_{(p,1]}  (x) , \quad\text{for all $s\in\mathcal{S}$.}
      $$ 
   Then 
        \begin{align*}
\inf_{Q=\left(Q_s\right)_{s\in\mathcal{S}}\in\mathcal{Q}_\alpha^p}\mathcal{L}(Q)=\mathcal{L}(Q^*). 
        \end{align*} 
  Moreover, if $\alpha\leq \sum\limits_{s\in\mathcal{S}}p_sQ_{f^*|s}(p)$ then $Q^*\in\mathcal{Q}^p_\alpha$ which means $Q^*$ is optimal solution of $(\mathcal{P}')$. 
  Otherwise, $(\mathcal{P}')$ does not admit a minimum. In addition  
$$|\mathcal{L}\left(Q^{\xi}\right)-\mathcal{L}\left(Q^{*})\right|=O(\xi)\quad \text{(as $\xi\to0$)},$$ where for $\xi>0$, $Q^{\xi}=(Q^{\xi}_s)_{s\in\mathcal{S}}$ is such that for each $s\in\mathcal{S}$
        \begin{align*}
            Q^\xi_s(x)=
        \min\{\alpha, Q_{f^*|s}(x)\}\mathds{1}_{(0,p]}(x)+
        \max\{\alpha+\xi,\sum_{s\in\mathcal{S}}p_sQ_{f^*|s}(x)\}\mathds{1}_{(p,1]}(x).
        \end{align*}
\end{lemma}
\begin{proof}
For each $\left(Q_s\right)_{s\in\mathcal{S}}\in\mathcal{Q}_{\alpha}^p$, we have
\begin{align}\label{epress:mse}
    \sum_{s\in\mathcal{S}}p_s\int_0^1\left( Q_{f^*|s}(x)-Q_s(x) \right)^2 dx&=\sum_{s\in\mathcal{S}}p_s\left[\int_0^p(Q_{f^*|s}(x)-Q_s(x))^2dx+\int_p^1(Q_{f^*|s}(x)-Q_s(x))^2dx \right]\nonumber\\
&=\sum_{s\in\mathcal{S}}p_s\left[\int_0^p(Q_{f^*|s}(x)-Q^\mathbf{-}_s(x))^2dx+\int_p^1(Q_{f^*|s}(x)-Q^\mathbf{+}(x))^2dx \right]\nonumber\\
&=\sum_{s\in\mathcal{S}}p_s\int_0^p(Q_{f^*|s}(x)-Q^\mathbf{-}_s(x))^2dx+\sum_{s\in\mathcal{S}}p_s\int_p^1(Q_{f^*|s}(x)-Q^\mathbf{+}(x))^2dx\nonumber\\
&=\sum_{s\in\mathcal{S}}p_s\int_0^p(Q_{f^*|s}(x)-Q^\mathbf{-}_s(x))^2dx+\int_p^1\left(\sum_{s\in\mathcal{S}}p_sQ_{f^*|s}(x)-Q^\mathbf{+}(x)\right)^2dx+\nonumber\\
&\qquad +\int_p^1\sum_{s\in\mathcal{S}}p_sQ_{f^*|s}^2(x)dx-\int_p^1\left(\sum_{s\in\mathcal{S}}p_sQ_{f^*|s}(x)\right)^2dx.
\end{align} Hence, solving 
 $(\mathcal{P}')$ is equivalent to solving $$\sum_{s\in\mathcal{S}}p_s\inf_{Q_s^\mathbf{-}\in\mathcal{Q}^-_{p,\alpha}}\int_0^p(Q_{f^*|s}(x)-Q^\mathbf{-}_s(x))^2dx+\inf_{Q^\mathbf{+}\in\mathcal{Q}^+_{p,\alpha}}\int_p^1\left(\sum_{s\in\mathcal{S}}p_sQ_{f^*|s}(x)-Q^\mathbf{+}(x)\right)^2dx,$$ where \begin{align*}
     \mathcal{Q}^-_{p,\alpha}&=\{f:[0,p]\to(-\infty,\alpha]: \text{$f$ is left continuous and non-decreasing} \};\\
      \mathcal{Q}^+_{p,\alpha}&=\{g:(p,1]\to(\alpha,+\infty): \text{$g$ is left continuous and non-decreasing} \}.
 \end{align*}
 For each $s\in\mathcal{S}$, if $Q_{f^*|s}(p)\leq \alpha$ then due to the strictly monotonicity of $Q_{f^*|s}$ we can get that
 $$Q_{f^*|s}\mathds{1}_{(0,p]}=\arg\inf_{Q_s^\mathbf{-}\in\mathcal{Q}^-_{p,\alpha}}\int_0^p(Q_{f^*|s}(x)-Q^\mathbf{-}_s(x))^2dx.$$  
    \begin{figure}[ht]
        \centering
        \includegraphics[width=0.45\linewidth]{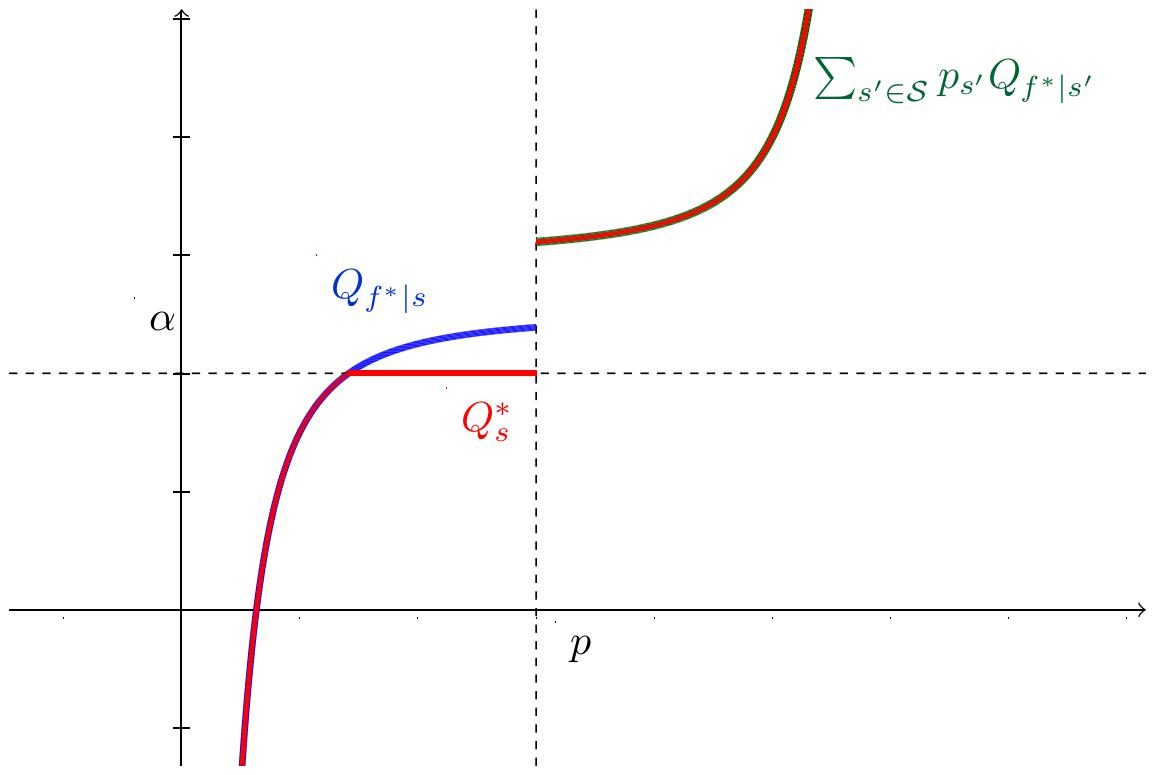}
        \caption{Optimal solution of $\mathcal{(P')}$}
        \label{fig:enter-label}
    \end{figure}
 Otherwise, in the case where $Q_{f^*|s}(p)>\alpha $, we need to introduce the following notation 
 $$
 x_s=\inf\{x|Q_{f^*|s}(x)>\alpha\}<p.
 $$ 
 Since $Q_{f^*|s}$ is strictly increasing, $Q_{f^*|s}(x)>\alpha$ for all $x\in(x_s,p]$. 
 Then for all $Q_s^-\in\mathcal{Q}^-_{\alpha,p}$ we have 
 \begin{align*}
        \int_0^p(Q_{f^*|s}(x)-Q^\mathbf{-}_s(x))^2dx&=\int_0^{x_s}(Q_{f^*|s}(x)-Q^\mathbf{-}_s(x))^2dx+\int_{x_s}^p(Q_{f^*|s}(x)-Q^\mathbf{-}_s(x))^2dx\\
        &\geq \int_{x_s}^p(Q_{f^*|s}(x)-\alpha)^2dx,
    \end{align*} 
    and the equality holds if (see Figure~\ref{fig:enter-label})
    $$
    Q_s^-=Q_{f^*|s}\mathds{1}_{[0,x_s]}+\alpha\mathds{1}_{(x_s,p]}.
    $$ 
To sum up, 
$$
\min\{\alpha,Q_{f^*|s}\}\mathds{1}_{(0,p]}=\underset{Q_s^\mathbf{-}\in\mathcal{Q}^-_{p,\alpha}}{\arg\inf}\int_0^p(Q_{f^*|s}(x)-Q^\mathbf{-}_s(x))^2dx.
$$ 
On the other hand, for all $Q\in\mathcal{Q}^+_{p,\alpha}$, we have 
$$
\int_p^1\left(\sum_{s\in\mathcal{S}}p_sQ_{f^*|s}(x)-Q^\mathbf{+}(x)\right)^2dx\geq\int_p^1\left(\sum_{s\in\mathcal{S}}p_sQ_{f^*|s}(x)-Q^*_s(x)\mathds{1}_{(p,1]}\right)^2dx ,
$$
where $Q_s^* = \min\{\alpha, Q_{f^*|s}\}\mathds{1}_{[0,p]}  +
        \max\{\alpha,\sum_{s\in\mathcal{S}}p_sQ_{f^*|s}\}\mathds{1}_{(p,1]}  , ~\text{for all $s\in\mathcal{S}$.}$ Thus, $ \mathcal{L}\left(\left(Q_s\right)_{s\in\mathcal{S}} \right)\geq\mathcal{L}\left(\left(Q^*_s\right)_{s\in\mathcal{S}} \right)$ for all $\left(Q_s\right)_{s\in\mathcal{S}}\in\mathcal{Q}^p_\alpha$. 
        Let us define (See Figure~\ref{fig:e})
\begin{equation}
\label{e-optimal}
Q^\xi_s(x)=
        \min\{\alpha, Q_{f^*|s}(x)\}\mathds{1}_{(0,p]}+
        \max\{\alpha+\xi,\sum_{s\in\mathcal{S}}p_sQ_{f^*|s}(x)\}\mathds{1}_{(p,1]},\quad\text{for all $s\in\mathcal{S}$.}     
\end{equation} 
It is clear that $\left(Q^\xi_s\right)_{s\in\mathcal{S}}\in \mathcal{Q}^p_\alpha$ for all $\xi>0$.
        \begin{figure}[ht]
            \centering
            \includegraphics[width=0.4\linewidth]{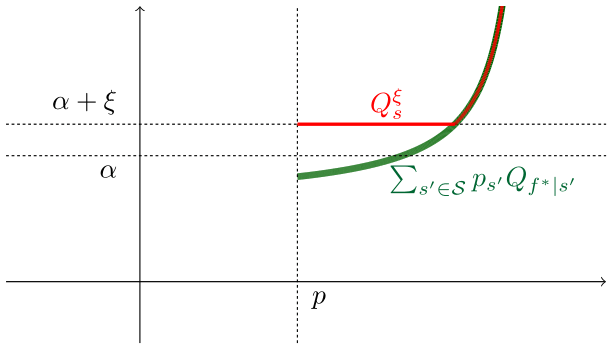}
            \caption{$\xi-$optimal solution}
            \label{fig:e}
        \end{figure} 
        Thanks to the Eq~\eqref{epress:mse}, we can deduce that
   \begin{align*}
       &\quad \left|\mathcal{L}\left(\left(Q^\xi_s\right)_{s\in\mathcal{S}} \right)-\mathcal{L}\left(\left(Q^*_s\right)_{s\in\mathcal{S}} \right)\right|\\\leq&\quad  \int_p^1\left| \left(\sum_{s\in\mathcal{S}}p_sQ_{f^*|s}(x)-\max\left\{ \alpha+\xi,\sum_{s\in\mathcal{S}}p_sQ_{f^*|s}(x)\right\}\right)^2-\left(\sum_{s\in\mathcal{S}}p_sQ_{f^*|s}(x)-\max\left\{ \alpha,\sum_{s\in\mathcal{S}}p_sQ_{f^*|s}(x)\right\}\right)^2\right| dx.
   \end{align*} For $\xi>0,$ denote  \begin{align*}
      t_\alpha &=\sup\left\{ t\in (0,1): \sum_{s\in\mathcal{S}}p_sQ_{f^*|s} (t)\leq \alpha\right\},\quad t^\xi_\alpha =\sup\left\{ t\in (0,1): \sum_{s\in\mathcal{S}}p_sQ_{f^*|s} (t)\leq \alpha+\xi\right\}.
  \end{align*} Let us consider the three following cases.\\
If $t_\alpha\leq t_\alpha^\xi\leq p$ then it is easy to see that
   \begin{align*}
      \left|\mathcal{L}\left(\left(Q^\xi_s\right)_{s\in\mathcal{S}} \right)-\mathcal{L}\left(\left(Q^*_s\right)_{s\in\mathcal{S}} \right)\right|=0.
   \end{align*} If $t_\alpha\leq p\leq t_\alpha^\xi$ then \begin{align*}
        \left|\mathcal{L}\left(\left(Q^\xi_s\right)_{s\in\mathcal{S}} \right)-\mathcal{L}\left(\left(Q^*_s\right)_{s\in\mathcal{S}} \right)\right|\leq \int_p^{t^\xi_\alpha}\left| \sum_{s\in\mathcal{S}}p_sQ_{f^*|s}(x)-(\alpha+\xi) \right|^2 dx\leq \xi^2(t^\xi_\alpha-p)\leq \xi^2
   \end{align*} The second equality holds since $\alpha< \sum_{s\in\mathcal{S}}p_sQ_{f^*|s}(x)\leq \alpha+\xi$ for all $x\in(p,t^\xi_\alpha)$.\\ Otherwise, we have \begin{align*}
       \left|\mathcal{L}\left(\left(Q^\xi_s\right)_{s\in\mathcal{S}} \right)-\mathcal{L}\left(\left(Q^*_s\right)_{s\in\mathcal{S}} \right)\right|&\leq \int_p^{t_\alpha}\left| \left(\alpha+\xi-\sum_{s\in\mathcal{S}}p_sQ_{f^*|s}(x)\right)^2-\left(\alpha-\sum_{s\in\mathcal{S}}p_sQ_{f^*|s}(x)\right)^2\right|dx+\\
        &\quad + \int_{t_\alpha}^{t^\xi_\alpha}\left| \alpha+\xi-\sum_{s\in\mathcal{S}}p_sQ_{f^*|s}(x)\right|^2 dx\\
        &\leq \int_p^{t_\alpha}\left|2\xi\left(\alpha-\sum_{s\in\mathcal{S}}p_sQ_{f^*|s}(x) \right)+\xi^2 \right|dx+\xi^2(t^\xi_\alpha -t_\alpha)\\
        &\leq \int_p^{t_\alpha}\left( 2\xi\left| \alpha-\sum_{s\in\mathcal{S}}p_sQ_{f^*|s}(p) \right|+\xi^2\right)dx+\xi^2\\
        &\leq \left( 2\xi\left| \alpha-\sum_{s\in\mathcal{S}}p_sQ_{f^*|s}(p) \right|+\xi^2\right) (t_\alpha-p)+\xi^2\\&\leq 2\xi\left| \alpha-\sum_{s\in\mathcal{S}}p_sQ_{f^*|s}(p) \right|+2\xi^2. 
   \end{align*} Thus we get that  $ \mathcal{L}\left(\left(Q^\xi_s\right)_{s\in\mathcal{S}} \right)\to\mathcal{L}\left(\left(Q^*_s\right)_{s\in\mathcal{S}} \right)$ as $\xi\to 0$.  Therefore, $$\inf_{Q=\left(Q_s\right)_{s\in\mathcal{S}}\in\mathcal{Q}_\alpha^p}\mathcal{L}(Q)=\mathcal{L}(Q^*).$$
\end{proof}
From Lemma \ref{le:sol}, it can be seen that $(\mathcal{P}')$ does not always admit a minimum, but an infimum. 
The sequence  $(Q^\xi)_{\xi}$ given by (\ref{e-optimal}) is called sequence of $\xi-$optimal solution of $(\mathcal{P}')$, and its $L_2-$limit will be the optimal solution of $(\mathcal{P}')$ in the case $\alpha\leq \sum\limits_{s\in\mathcal{S}}p_sQ_{f^*|s}(p)$.

\begin{proof}[Proof of Theorem \ref{thr:optimalPredict}] 
 Let $U\sim\mathcal{U}[0,1]$, and denote by $\mu_s$ the distribution  of $Q^{*}_s(U)$ for each $s\in\mathcal{S}$. Then, similar to the proof of Eq:(\ref{eq:uniform})
, we can prove that $Q^{*}_s$ is the quantile function of $\mu_s$.  For all $s\in\mathcal{S}$ and $t\in\mathbb{R}$, we have that $$\begin{aligned}
\mathbb{P}(g^{*}_{\alpha,p}(X, S) \leq t|S=s) & =\mathbb{P}\left(Q^*_s \circ F_{f^* \mid s} \circ f^*(X, s) \leq t|S=s\right) \\
& =\mathbb{P}\left(f^*(X, s) \leq Q_{f^*\mid s} \circ F_{\nu_s}(t)|S=s\right)\\
&=F_{f^*|s}(Q_{f^*|s}(F_{\nu_s}(t)))\\
&=F_{\nu_s}(t),
\end{aligned}
$$which shows that $\left(\nu_{g^{*}_{\alpha,p}}\right)_{s\in\mathcal{S}}=(\mu_s)_{s\in\mathcal{S}}$. From Lemma \ref{le:sol}, \begin{align*}
  \mathbb{E}\left[(f^*(X,S)-g^*_{\alpha,p}(X,S))^2\right] &=\sum_{s\in\mathcal{S}}p_s \mathbb{E}\left[(f^*(X,s)- Q^*_s \circ F_{f^* \mid s} \circ f^*(X, s) )^2|S=s\right]\\&=\sum_{s\in\mathcal{S}}p_s\mathcal{W}\left(\nu_{f^*|s},\nu_{g^{*}_{\alpha,p}|s}\right)\\&=\sum_{s\in\mathcal{S}}p_s\int_0^1 \left(Q_{f^*|s}(x)-Q^{*}_s(x) \right)^2 dx\\
    &=\inf_{\left(Q_s\right)_{s\in\mathcal{S}}\in\mathcal{Q}_\alpha^p}  \sum_{s\in\mathcal{S}}p_s\int_0^1\left( Q_{f^*|s}(x)-Q_s(x) \right)^2 dx\\
    &= \inf_{g\in \mathcal{F}^p_\alpha}\mathbb{E}\left[(f^*(X,S)-g(X,S))^2\right].
\end{align*}
Similarly, we get that
$$\mathbb{E}\left[(f^*(X,S)-g^{*,\xi}_{\alpha,p}(X,S))^2\right]=\sum_{s\in\mathcal{S}}p_s\mathcal{W}\left(\nu_{f^*|s},\nu_{g^{*,\xi}_{\alpha,p}|s}\right)=\sum_{s\in\mathcal{S}}p_s\int_0^1\left( Q_{f^*|s}(x)-Q^{*,\xi}_s(x) \right)^2 dx.$$ Thus, from Lemma \ref{le:sol} we can deduce Theorem~\ref{thr:optimalPredict}.
\end{proof}

\begin{proof}[Proof of Remark \ref{erisk bound}]
   \begin{align*}
\mathbb{E}\left|g^*_{\alpha,p}(X,S)-g^{*,\xi}_{\alpha,p}(X,S)\right|^2
&=\sum_{s\in\mathcal{S}}p_s\mathbb{E}\left|g^*_{\alpha,p}(X,s)-g^{*,\xi}_{\alpha,p}(X,s)|S=s\right|^2\\
&=\sum_{s\in\mathcal{S}}p_s\mathbb{E}\left|Q_s^*\circ F_{f^*|s}\circ f^*(X,s)-Q_s^\xi\circ F_{f^*|s}\circ f^*(X,s)|S=s\right|^2.
   \end{align*}  For $s\in \mathcal{S}$ and $x\in\mathbb{R}^d$, we have \begin{align*}
       &\quad \left|Q_s^*\circ F_{f^*|s}\circ f^*(x,s)-Q_s^\xi\circ F_{f^*|s}\circ f^*(x,s)\right|^2\\
       =&\quad \left|\min\left\{ \alpha, Q_{f^*|s}(F_{f^*|s}(f^*(x,s))\right\}\mathds{1}_{[0,p]}(F_{f^*|s}(f^*(x,s))\right.\\&+ \left.     \max\left\{\alpha,\sum\limits_{s\in\mathcal{S}}p_sQ_{f^*|s}(F_{f^*|s}(f^*(x,s))\right\}\mathds{1}_{(p,1]}(F_{f^*|s}(f^*(x,s))\right.\\&-\left.\min\left\{ \alpha, Q_{f^*|s}(F_{f^*|s}(f^*(x,s))\right\}\mathds{1}_{[0,p]}(F_{f^*|s}(f^*(x,s))\right.\\&-     \left.\max\left\{\alpha+\xi,\sum\limits_{s\in\mathcal{S}}p_sQ_{f^*|s}(F_{f^*|s}(f^*(x,s))\right\}\mathds{1}_{(p,1]}(F_{f^*|s}(f^*(x,s))\right|^2\\
      = &\quad \left|\left(\max\left\{\alpha,\sum\limits_{s\in\mathcal{S}}p_sQ_{f^*|s}(F_{f^*|s}(f^*(x,s))\right\}-\max\left\{\alpha+\xi,\sum\limits_{s\in\mathcal{S}}p_sQ_{f^*|s}(F_{f^*|s}(f^*(x,s))\right\}\right)\right|^2\mathds{1}_{(p,1]}(F_{f^*|s}(f^*(x,s))\\
      &\leq \xi^2 \mathds{1}_{(p,1]}(F_{f^*|s}(f^*(x,s)),
   \end{align*} the last inequality holds since $|\max\{u,a\}-\max\{v,a\}|\leq |u-v|.$ Thus,
   \begin{align*}
       \mathcal{E} (g^{*,\xi}_{\alpha,p})\leq \sum_{s\in\mathcal{S}}p_s\xi^2\mathbb{P}(F_{f^*|s}(f^*(X,s))>p|S=s)=\xi^2\sum_{s\in\mathcal{S}}p_s\mathbb{P}(F_{f^*|s}(f^*(X,s))>p|S=s)\leq \xi^2.
   \end{align*}
\end{proof}

\section{Proofs of unfairness and risk controls} 
\label{app:proofEmpirical}
This part of the appendix is devoted to the proofs of the theoretical guarantees of $\hat{g}^{\xi}_{\alpha,p}$ -- for short we will always write $\hat{g}, \hat{Q}_s, g^*, g^{*\xi}$ instead of $\hat{g}^{\xi}_{\alpha,p}, \hat{Q}^{\xi}_{s}, g^*_{\alpha,p} , g^{*\xi}_{\alpha,p} $, respectively. To that end, we first start by providing some technical tools.
\subsection{Technical tools}
Firstly, we will introduce the definition of dual generalized inverse of a nondecreasing left-continuous function \cite{wacker2023pleasetextitanothernotegeneralized}.
\begin{definition}[Dual Generalized inverse of a nondecreasing left-continuous function]
 Let $f$ be a real-valued, nondecreasing, left continuous function defined on the open interval $(a, b)$ where $-\infty \leq a<b \leq \infty$. Then the dual generalized inverse of $f$ is defined by
$$
f^{-1}(y)=\inf\{ x \in(a, b): f(x)> y\}=\sup \{x \in(a, b): f(x) \leq y\}.
$$
for $-\infty<y<\infty$ (with the convention $\sup (\emptyset)=a$).
\end{definition}
Now we state a central tool to get our convergence rates, the Dvoretzky–Kiefer–Wolfowitz inequality~\citep[Corollary 1]{massart1990tight}.
\begin{lemma}[Dvoretzky–Kiefer–Wolfowitz (DKW) inequality]
    \label{thm:DKW}
    Let $Z_1, \ldots, Z_n$ be \emph{i.i.d.} real valued random variables with cumulative distribution $F$.
    Let $\hat F$ be the empirical cumulative distribution of $Z_1, \ldots, Z_n$, then
    \begin{align*}
        \mathbb{E}  \left\|\hat{F} - {F} \right\|_{\infty}  \leq \sqrt{\frac{\pi}{2n}}\enspace,
    \end{align*}
    where $\left\|\hat{F} - {F} \right\|_{\infty}  = \sup_{t \in \mathbb{R}} \left| \hat{F}(t) - {F}(t) \right|$.
\end{lemma}

\subsection{Proof of Theorem~\ref{thm:unfairnessControl} -- Unfairness control}
\label{app:proofUnfair}
\begin{proof}
    Let $X^s \sim \mathbb{P}_{X \mid S=s}$ and $X^s$ be independent from $\mathcal{D}= \mathcal{D}_n \cup \mathcal{D}_N \cup \mathcal{E}$ --- that is the labeled, unlabeled data, and the noise variables $\varepsilon_{i s}, \varepsilon$ --- then it holds that
$$
\mathbb{P}(\hat{g}(X, S) \leq t \mid S=s)=\mathbb{P}\left(\hat{g}\left(X^s, s\right) \leq t\right) \enspace .
$$   
\paragraph{Proof of Eq~\eqref{eq:f-pf}:} For all $s\in\mathcal{S}$, we have that \begin{align*}
    \mathbb{P}(\hat{g}(X^s,s)\leq \alpha )= \mathbb{P}\left(\hat{F}_{\hat{f}|s}(\hat{f}(X^s,s)+\varepsilon) \leq \left(\hat{Q}_s\right)^{-1}(\alpha) \right) \enspace,
\end{align*} and it is clear that
\begin{align*}
   \left(\hat{Q}_s\right)^{-1}(\alpha)=\inf\left\{ y\in (p,1)|~\max\left\{\alpha+\xi,\sum\limits_{s\in\mathcal{S}}\hat{p}_s\hat{Q}_{\hat{f}|s}(y)\right\}> \alpha \right\} = \inf\left\{ y\in (p,1)\right\}=p.
\end{align*} 
Hence,
$$\mathbb{P}(\hat{g}(X^s,s)\leq \alpha )= \mathbb{P}\left(\hat{F}_{\hat{f}|s}(\hat{f}(X^s,s)+\varepsilon) \leq p \right).$$ 
Moreover, we also have $ \mathbb{P}\left(F_{\hat{f}|s}(\hat{f}(X^s,s)+\varepsilon) \leq p \right) = p$ since $F_{\hat{f}|s}(\hat{f}(X^s,s)+\varepsilon)$ is, conditional on all datasets, uniformly distributed. Then 
\begin{eqnarray*}
\left| \mathbb{P}(\hat{g}(X^s,s)\leq \alpha )- p \right|  & = & 
\left|\mathbb{P}\left(\hat{F}_{\hat{f}|s}(\hat{f}(X^s,s)+\varepsilon) \leq p \right) - \mathbb{P}\left(F_{\hat{f}|s}(\hat{f}(X^s,s)+\varepsilon) \leq p \right)\right|
\\
& \leq &
\sup_{t\in (0,1)} 
\left|\mathbb{P}\left(\hat{F}_{\hat{f}|s}(\hat{f}(X^s,s)+\varepsilon) \leq t \right) - \mathbb{P}\left(F_{\hat{f}|s}(\hat{f}(X^s,s)+\varepsilon) \leq t \right)\right|
\end{eqnarray*}
For all $s\in \mathcal{S}$, denote by $U_{n,s}(t)$ the empirical process
\begin{equation}
    \label{eq:EmpProc}
    U_{n,s}(t) = \left|\mathbb{P}\left(\hat{F}_{\hat{f}|s}(\hat{f}(X^s,s)+\varepsilon) \leq t \right) - \mathbb{P}\left(F_{\hat{f}|s}(\hat{f}(X^s,s)+\varepsilon) \leq t \right)\right|\enspace,
\end{equation}
so that we need to bound $\sup_{t\in (0,1)} U_{n,s}(t)$. For all $t\in(0,1)$ and conditionally to all datasets
\begin{eqnarray*}
    U_{n,s}(t) & \leq & \mathbb{P}\left(  \left|    F_{\hat{f}|s}(\hat{f}(X^s,s)+\varepsilon) - t\right|  \leq  \left| \hat{F}_{\hat{f}|s}(\hat{f}(X^s,s)+\varepsilon)  - F_{\hat{f}|s}(\hat{f}(X^s,s)+\varepsilon)  \right|\right) \\
    & \leq & 
    \mathbb{P}\left(  \left|    F_{\hat{f}|s}(\hat{f}(X^s,s)+\varepsilon) - t\right|  \leq  \left\| \hat{F}_{\hat{f}|s} - F_{\hat{f}|s}   \right\|_{\infty} \right)
    \leq 
     2 \left\| \hat{F}_{\hat{f}|s} - F_{\hat{f}|s}   \right\|_{\infty} \enspace,
\end{eqnarray*}
where we again used the fact that $F_{\hat{f}|s}(\hat{f}(X^s,s)+\varepsilon)$ is uniformly distributed. Take the supremum and then the expectation with respect to the data in $\mathcal{D}$ and apply the DKW inequality, recalled in Lemma~\ref{thm:DKW}, to get
\begin{equation}
    \label{eq:BoundProcess}
    \mathbf{E} \left[ \sup_{t\in (0,1)}  U_{n,s}(t) \right] \leq \frac{C}{\sqrt{p_sN}}\enspace,
\end{equation}
for some absolute constant $C >0$. Finally, applying Lemma 4.1 in \citep{GyorfyBook2002} yields the desired result.

\paragraph{Proof of Eq:(\ref{eq:e-pf}):} Due to the fact that for $t\geq \alpha$ 
\begin{align*}
  \hat{g}\left(X^s, s\right) \leq t~\Longleftrightarrow \hat{F}_{\hat{f}|s}(\hat{f}(x,s)+\varepsilon)\leq \left(\hat{Q}_s\right)^{-1}(t) \enspace,
\end{align*} 
we can get that 
\begin{align}\label{pf:1}
   &\quad  \sup_{t\geq  \alpha} \left|\mathbb{P}(\hat{g}(X^s, s) \leq t)-\mathbb{P}\left(\hat{g}(X^{s'}, s') \leq t \right)\right|\nonumber
   \\
    \leq 
    &
    \quad  \sup_{t\geq \alpha} \left|\mathbb{P}\left(\hat{F}_{\hat{f}|s}(\hat{f}(X^s,s)+\varepsilon) \leq \left(\hat{Q}_s\right)^{-1}(t) \right)-\mathbb{P}\left(\hat{F}_{\hat{f}|s'}(\hat{f}(X^{s'},s')+\varepsilon)\leq \left(\hat{Q}_{s'}\right)^{-1}(t)\right)\right|\enspace.
\end{align} 
It can be seen that, 
\begin{align*}
    \left( \hat{Q}_s \right)^{-1}(t)=\inf\left\{ y\in (0,1)|~ \hat{Q}_s (y)> t \right\} \enspace.
\end{align*} 
Due to the fact that $ \hat{Q}_s=\min\left\{\alpha,\hat{Q}_{\hat{f}|s} \right\}\mathds{1}_{[0,p]}+\max\left\{\alpha+\xi,\sum\limits_{s\in\mathcal{S}}\hat{p}_s\hat{Q}_{\hat{f}|s} \right\}\mathds{1}_{(p,1]}$ for each $s\in\mathcal{S}$, we can deduce that for all $t\geq\alpha$ \begin{align*}
    \left( \hat{Q}_s \right)^{-1}(t)=\inf\left\{ y\in (0,1)|~ \hat{Q}_s (y)> t \right\}=\inf\left\{ y\in (p,1)|~\max\left\{\alpha+\xi,\sum\limits_{s\in\mathcal{S}}\hat{p}_s\hat{Q}_{\hat{f}|s}(y)\right\}> t \right\} \enspace,
\end{align*} 
which does not depend on the value of $s$. Hence, we can deduce that for $t\geq\alpha$,
$\left( \hat{Q}_s\right)^{-1}(t)=\left( \hat{Q}_{s'} \right)^{-1}(t)$  for all $s,s'\in\mathcal{S}.$ Thus, from \eqref{pf:1} we get that 
\begin{align*}
   &\quad  \sup_{t\geq  \alpha} \left|\mathbb{P}(\hat{g}(X^s, s) \leq t)-\mathbb{P}\left(\hat{g}(X^{s'}, s') \leq t \right)\right|\nonumber
   \\
    \leq&
    \quad \sup_{t\in[p,1]} \left|\mathbb{P}\left(\hat{F}_{\hat{f}|s}(\hat{f}(X^s,s)+\varepsilon)\leq t\right)-\mathbb{P}\left(\hat{F}_{\hat{f}|s'}(\hat{f}(X^{s'},s')+\varepsilon)\leq t\right)\right|\\  \leq&
    \quad \sup_{t\in[0,1]} \left|\mathbb{P}\left(\hat{F}_{\hat{f}|s}(\hat{f}(X^s,s)+\varepsilon)\leq t\right)-\mathbb{P}\left(\hat{F}_{\hat{f}|s'}(\hat{f}(X^{s'},s')+\varepsilon)\leq t\right)\right|
    \enspace.
\end{align*} 
Here again, we use the fact that ${F}_{\hat{f}|s}(\hat{f}(X^s,s)+\varepsilon$ is uniformly distributed for all $s\in \mathcal{S}$ to write that for all $t\in (0,1)$, we have $\mathbb{P}\left({F}_{\hat{f}|s}(\hat{f}(X^s,s)+\varepsilon)\leq t\right)=t=\mathbb{P}\left({F}_{\hat{f}|s}(\hat{f}(X^{s},s)+\varepsilon)\leq t\right)$ so that using the triangle inequality, we can write
\begin{equation*}
     \mathbf{E} \sup_{t\geq  \alpha} \left|\mathbb{P}(\hat{g}(X^s, s) \leq t)-\mathbb{P}\left(\hat{g}(X^{s'}, s') \leq t \right)\right| \leq 2 \mathbf{E}  \sup_{t\in (0,1)} U_{n,s}(t) \leq \frac{2C}{\sqrt{N_s}}\enspace,
\end{equation*}
where $U_{n,s}(t) $ is defined in Equation~\eqref{eq:EmpProc} and the last equality comes from~\eqref{eq:BoundProcess}.

\end{proof}

\subsection{Proof of Theorem~\ref{thm:riskBound} -- Risk control}
\label{app:proofRisk}
Before providing the Proof of Theorem~\ref{thm:riskBound}, we give a technical lemma.
\begin{lemma}
\label{lem:lemRiskCDF}
Let $t \in \mathbb{R}$, for each $s \in \mathcal{S}$, the following holds
\begin{equation*}
\left[\left|F_{\hat{f}|s}(t) -F_{f^*|s}(t) \right|\right]  \leq C\left( \left\|\hat{f}-f^*\right\|_{\infty} + \varepsilon \right) \qquad {\it a.s.} .  
\end{equation*}
\end{lemma}
\begin{proof}
For each $t \in \mathbb{R}$, and $s \in \mathcal{S}$, we have that
\begin{equation*}
\left|F_{\hat{f}|s}(t) -F_{f^*|s}(t)\right|  = \left|\mathbb{E}\left[\one_{\left\{\hat{f}(X,s)+ \varepsilon \leq t \right\}} -\one_{\left\{f^*(X,s) \leq t\right\}} \right]\right|
\leq \mathbb{E}\left[\one_{\left\{\left|\hat{f}(X,s)+ \varepsilon-f^*(X,s)\right| \geq \left|f^*(X,s)-t  \right|   \right\}}\right].
\end{equation*}
From the above inequality, we deduce that
\begin{equation*}
 \left|F_{\hat{f}|s}(t) -F_{f^*|s}(t)\right| \leq    
 \mathbb{E}\left[\one_{\left\{\left\|f^*-\hat{f} \right\|_{\infty}+ \varepsilon \geq \left|f^*(X,s)-t  \right|   \right\}}\right].
\end{equation*}
Finally, under Assumpion~\ref{as1}, we deduce that, conditional on $\mathcal{D}_n$, and $\mathcal{D}_N$ that
\begin{equation*}
\mathbb{P}\left(\left|f^*(X,s)-t  \right|  \leq \left\|f^*-\hat{f} \right\|_{\infty}+ \varepsilon \right) \leq C \left(\left\|f^*-\hat{f} \right\|_{\infty}+ \varepsilon \right),  
\end{equation*}
which gives the result.
\end{proof}

\begin{proof}[Proof of Theorem \ref{thm:riskBound}]
First, we remind the defintion of the predictor $\hat{g}^{\xi}$
  \begin{align}
     \hat{g}(x,s)=\hat{Q}_s\circ \hat{F}_{\hat{f}|s}\circ (\hat{f}(x,s)+\varepsilon)
\end{align} where $\varepsilon\sim \mathcal{U}[-\sigma,\sigma]$ and
\begin{align}
\hat{Q}_s(x)=\left\{\begin{array}{ll}
        \min\{\alpha,\hat{Q}_{\hat{f}|s}(x)\} &\text{if $x\in[0,p]$} , \\
\max\left\{\alpha+\xi, \sum\limits_{s\in\mathcal{S}}\hat{p}_s\hat{Q}_{\hat{f}|s}(x)\right\} &\text{if $x\in(p,1]$}.
    \end{array}\right.
\end{align}
We start with the following decomposition 
\begin{align*}
    \mathbf{E}\left[\left|g^*(X,S)-\hat{g}(X,S)\right|\right]&\leq  \mathbf{E}\left[\left|g^*(X,S)-g^{*,\xi}(X,S))\right|\right]+\mathbf{E}\left[\left|g^{*,\xi}(X,S)-\hat{g}(X,S)\right|\right]\\
    &
\end{align*}
From Equation~\eqref{eq:eqExcessRiskXi}, we then have
\begin{equation*}
 \mathbf{E}\left[\left|g^*(X,S)-\hat{g}(X,S)\right|\right] \leq \xi+\mathbf{E}\left[\left|g^{*,\xi}(X,S)-\hat{g}(X,S)\right|\right].
\end{equation*}
Now we consider the second term in the {\it r.h.s.} of the above equation.
We have 
\begin{align*}
    \E[|g^{*,\xi}(X,S)-\hat{g}(X,S)|]&=\sum_{s\in\mathcal{S}}p_s\E[|g^{*,\xi}(X,s)-\hat{g}(X,s)||S=s],
\end{align*}
and let us define
\begin{align*}
    \Bar{Q}_s(x)=\left\{\begin{array}{ll}
        \min\{\alpha,Q_{f^*|s}(x)\} &\text{if $x\in[0,p]$} , \\
\max\left\{\alpha+\xi, \sum\limits_{s\in\mathcal{S}}\hat{p}_sQ_{f^*|s}(x)\right\} &\text{if $x\in(p,1]$}.
    \end{array}\right.
\end{align*}
From the above definition, we have

\begin{align*}
    \left|g^{*,\xi}(x,s)-\Bar{Q}_s\circ F_{f^*|s}\circ f^*(x,s) \right|&\leq\left|g^{*,\xi}(x,s)-\Bar{Q}_s\circ F_{f^*|s}\circ f^*(x,s) \right|\mathds{1}_{[0,p]}(F_{f^*|s}(f^*(x,s))\\
    &\qquad + \left|g^{*,\xi}(x,s)-\Bar{Q}_s\circ F_{f^*|s}\circ f^*(x,s) \right|\mathds{1}_{(p,1]}(F_{f^*|s}(f^*(x,s))\\
    & \leq \left| \max\left\{ \alpha+\xi,\sum_{s^{\prime}\in\mathcal{S}}p_{s^{\prime}}Q_{f^*|s^{\prime}}(F_{f^*|s^{\prime}}(f^*(x,s^{\prime}))\right\}\right.\\&\left.\qquad-\max\left\{ \alpha+\xi,\sum_{s\in\mathcal{S}}\hat{p}_{s^{\prime}}Q_{f^*|s^{\prime}}(F_{f^*|s^{\prime}}(f^*(x,s^{\prime}))\right\} \right|\\
    &\leq \sum_{s^{\prime}\in \mathcal{S}}|p_{s^{\prime}}-\hat{p_{s^{\prime}}}|\times \left|Q_{f^*|s^{\prime}}\circ F_{f^*|s^{\prime}}\circ f^*(x,s^{\prime}) \right|\\
    &\leq \frac{1}{2} \sum_{s^{\prime} \in \mathcal{S}} \underline{\lambda}_{s^{\prime}}^{-1}\left|p_{s^{\prime}}-\hat{p}_{s^{\prime}}\right| ,
\end{align*}    where the last inequality holds due tho the fact that $$ \mathbb{P}\left(\left|f^*(X, S)\right| \leq \frac{\underline{\lambda}_{s^{\prime}}^{-1}}{2} \mid S=s^{\prime}\right)=\int_{\left|f^*\left(x, s^{\prime}\right)\right| \leq \frac{\underline{\lambda}_{s^{\prime}}^{-1}}{2}} \mathbb{P}_{X \mid S=s^{\prime}}(d x)=\int_{|t| \leq \frac{\underline{\lambda}_{s^{\prime}}^{-1}}{2}} q_{s^{\prime}}(t) d t \geq \underline{\lambda}_{s^{\prime}} \int_{|t| \leq \frac{\underline{\lambda}_{s^{\prime}}^{-1}}{2}} d t=1.$$
Therefore, we deduce
\begin{align*}
    \E[|g^{*,\xi}(X,S)-\hat{g}(X,S)|]&=\sum_{s\in\mathcal{S}}p_s\E[|g^{*,\xi}(X,s)-\hat{g}(X,s)||S=s]\\
    &\leq \sum_{s\in\mathcal{S}}p_s\E[|\hat{Q}_s\circ \hat{F}_{\hat{f}|s}\circ (\hat{f}(X,s)+\varepsilon)-\hat{g}(X,s)-\Bar{Q}_s\circ F^*_{f^*|s}\circ f^*(X,s)||S=s]\\
    &\quad+\frac{1}{2} \sum_{s^{\prime} \in \mathcal{S}} \underline{\lambda}_{s^{\prime}}^{-1}\left|p_{s^{\prime}}-\hat{p}_{s^{\prime}}\right|.
\end{align*}
The above inequality yields, 
\begin{align*}
    \mathbf{E}[|g^{*,\xi}(X,S)-\hat{g}(X,S)|]
    &\leq \sum_{s\in\mathcal{S}}p_s  \mathbf{E}[|\hat{Q}_s\circ \hat{F}_{\hat{f}|s}\circ (\hat{f}(X,s)+\varepsilon)-\Bar{Q}_s\circ F^*_{f^*|s}\circ f^*(X,s)||S=s] \\ 
    &\quad+\frac{1}{2} \sum_{s^{\prime} \in \mathcal{S}} \underline{\lambda}_{s^{\prime}}^{-1}  \mathbf{E}\left[\left|p_{s^{\prime}}-\hat{p}_{s^{\prime}}\right|\right].
\end{align*}
It can be seen  that the term $\mathbf{E}\left|p_s-\hat{p}_s\right|=N^{-1} \mathbf{E}\left|N p_s-V\right|$, where $V$ is the binomial random variable with parameters $\left(N, p_s\right)$, thus using the Cauchy-Schwarz inequality we can write $\mathbf{E}\left|p_s-\hat{p}_s\right| \leq N^{-1} \sqrt{\operatorname{Var}(V)}=\sqrt{p_s\left(1-p_s\right) / N}$ and the above bound reads as
\begin{align}\label{eq:eqRisk0}
          \mathbf{E}[|g^{*,\xi}(X,S)-\hat{g}(X,S)|]
    &\leq \sum_{s\in\mathcal{S}}p_s  \mathbf{E}[|\hat{Q}_s\circ \hat{F}_{\hat{f}|s}\circ (\hat{f}(X,s)+\varepsilon)-\Bar{Q}_s\circ F^*_{f^*|s}\circ f^*(X,s)||S=s]\nonumber \\
    &\quad +\sqrt{|\mathcal{S}|} N^{-1 / 2}
\end{align}

Now, we consider the first term in {\it r.h.s.} of the above equation. 
The following decomposition holds
\begin{multline*}
\E[|\hat{Q}_s\circ \hat{F}_{\hat{f}|s}\circ (\hat{f}(X,s)+\varepsilon)-\Bar{Q}_s\circ F^*_{f^*|s}\circ f^*(X,s)||S=s] \\
   =\E\left[\left| \min\left\{\alpha,\hat{Q}_{\hat{f}|s}(\hat{F}_{\hat{f}|s}(\hat{f}(X,s)+\varepsilon)\right\}\mathds{1}_{[0,p]}(\hat{F}_{\hat{f}|s}(\hat{f}(X,s)+\varepsilon)\right.\right.\\
   \left.+ \max\left\{\alpha+\xi,\sum_{s'\in\mathcal{S}}\hat{p}_s\hat{Q}_{\hat{f}|s'}(\hat{F}_{\hat{f}|s'}(\hat{f}(X,s')+\varepsilon)\right\}\mathds{1}_{(p,1]}(\hat{F}_{\hat{f}|s}(\hat{f}(X,s)+\varepsilon)\right.\\
    -\left. \min\left\{\alpha,Q_{f^*|s'}({F}_{f^*|s}(f^*(X,s))\right\}\mathds{1}_{[0,p]}(F_{f^*|s}(f^*(X,s))\right.\\
    \left.\left.- \max\left\{\alpha+\xi,\sum_{s'\in\mathcal{S}}\hat{p}_{s'}Q_{f^*|s'}(F_{f^*|s'}({f}^*(X,s))\right\}\mathds{1}_{(p,1]}(F_{f^*|s'}(f^*(X,s))\right||S=s\right].\\  
\end{multline*}
Therefore, we deduce
\begin{multline*}
\E[|\hat{Q}_s\circ \hat{F}_{\hat{f}|s}\circ (\hat{f}(X,s)+\varepsilon)-\Bar{Q}_s\circ F^*_{f^*|s}\circ f^*(X,s)||S=s] \leq \\
\E\left[\left|\min\left\{\alpha,\hat{Q}_{\hat{f}|s}(\hat{F}_{\hat{f}|s}(\hat{f}(X,s)+\varepsilon)\right\}-\min\left\{\alpha,Q_{f^*|s'}({F}_{f^*|s}(f^*(X,s))\right\}\right|\mathds{1}_{[0,p]}(\hat{F}_{\hat{f}|s}(\hat{f}(X,s)+\varepsilon)|S=s \right]\\
    +\E\left[\left|\min\left\{\alpha,Q_{f^*|s'}({F}_{f^*|s}(f^*(X,s))\right\}\right|\times \left|\mathds{1}_{[0,p]}(\hat{F}_{\hat{f}|s}(\hat{f}(X,s)+\varepsilon)-\mathds{1}_{[0,p]}(F_{f^*|s}(f^*(X,s)) \right||S=s\right]
    + \\\E \left[\left| \max\left\{\alpha+\xi,\sum_{s'\in\mathcal{S}}p_s\hat{Q}_{\hat{f}|s'}(\hat{F}_{\hat{f}|s'}(\hat{f}(X,s)+\varepsilon)\right\}-\max\left\{\alpha+\xi,\sum_{s'\in\mathcal{S}}\hat{p}_{s'}Q_{f^*|s'}(F_{f^*|s'}({f}^*(X,s))\right\}\right| \mathds{1}_{(p,1]}(\hat{F}_{\hat{f}|s}(\hat{f}(X,s)+\varepsilon)|S=s\right]\\
 +\E\left[\left|\max\left\{\alpha+\xi,\sum_{s'\in\mathcal{S}}\hat{p}_{s'}Q_{f^*|s'}(F_{f^*|s'}(\hat{f}(X,s'))\right\}\right|\times \left|\mathds{1}_{(p,1])}(\hat{F}_{\hat{f}|s}(\hat{f}(X,s)+\varepsilon)-\mathds{1}_{(p,1]}(F_{f^*|s}(f^*(X,s)) \right||S=s\right].
\end{multline*}
From the above inequality, we get
\begin{multline*}
\E[|\hat{Q}_s\circ \hat{F}_{\hat{f}|s}\circ (\hat{f}(X,s)+\varepsilon)-\Bar{Q}_s\circ F^*_{f^*|s}\circ f^*(X,s)||S=s] \leq   
 \mathbf{E}\left[\left|\hat{Q}_{\hat{f}|s}(\hat{F}_{\hat{f}|s}(\hat{f}(X,s)+\varepsilon)-Q_{f^*|s}({F}_{f^*|s}(f^*(X,s))\right||S=s \right] \\+\left|\min\left\{\alpha,\underline{\lambda}_{s}^{-1}\right\}\right|\mathbf{E}\left[ \left|\mathds{1}_{[0,p]}(\hat{F}_{\hat{f}|s}(\hat{f}(X,s)+\varepsilon)-\mathds{1}_{[0,p]}(F_{f^*|s}(f^*(X,s)) \right||S=s\right]\\
   +\sum_{s'\in\mathcal{S}}p_{s'}\mathbf{E} \left[ \left| \hat{Q}_{\hat{f}|s'}(\hat{F}_{\hat{f}|s'}(\hat{f}(X,s)+\varepsilon)-\sum_{s'\in\mathcal{S}}Q_{f^*|s'}(F_{f^*|s'}({f}^*(X,s))\right| |S=s\right]\\
 +\max\left\{\alpha+\xi,\max_{s\in\mathcal{S}}\underline{\lambda}_{s^{\prime}}^{-1}\right\}\mathbf{E}\left[ \left|\mathds{1}_{(p,1])}(\hat{F}_{\hat{f}|s}(\hat{f}(X,s)+\varepsilon)-\mathds{1}_{(p,1]}(F_{f^*|s}(f^*(X,s)) \right||S=s\right].
\end{multline*}

Thus 
\begin{align*}
&\mathbf{E}[|\hat{Q}_s\circ \hat{F}_{\hat{f}|s}\circ (\hat{f}(X,s)+\varepsilon)-\Bar{Q}_s\circ F^*_{f^*|s}\circ f^*(X,s)||S=s]\\ & \leq \mathbf{E}\left[\left|\hat{Q}_{\hat{f}|s}(\hat{F}_{\hat{f}|s}(\hat{f}(X,s)+\varepsilon)-Q_{f^*|s}({F}_{f^*|s}(f^*(X,s))\right||S=s \right] \\&+\left|\min\left\{\alpha,\underline{\lambda}_{s}^{-1}\right\}\right|\mathbf{E}\left[ \left|\mathds{1}_{[0,p]}(\hat{F}_{\hat{f}|s}(\hat{f}(X,s)+\varepsilon)-\mathds{1}_{[0,p]}(F_{f^*|s}(f^*(X,s)) \right||S=s\right]\\
  & +\sum_{s'\in\mathcal{S}}p_{s'}\mathbf{E} \left[ \left| \hat{Q}_{\hat{f}|s'}(\hat{F}_{\hat{f}|s'}(\hat{f}(X,s)+\varepsilon)-\sum_{s'\in\mathcal{S}}Q_{f^*|s'}(F_{f^*|s'}({f}^*(X,s))\right| |S=s\right]\\
 &+\mathbf{E}\left[\left|\max\left\{\alpha+\xi,\sum_{s'\in\mathcal{S}}\hat{p}_{s'}\underline{\lambda}_{s^{\prime}}^{-1}\right\}\right|\right]\mathbf{E}\left[ \left|\mathds{1}_{(p,1])}(\hat{F}_{\hat{f}|s}(\hat{f}(X,s)+\varepsilon)-\mathds{1}_{(p,1]}(F_{f^*|s}(f^*(X,s)) \right||S=s\right],\\ 
\end{align*}
which yields
\begin{align*}
 &\mathbf{E}[|\hat{Q}_s\circ \hat{F}_{\hat{f}|s}\circ (\hat{f}(X,s)+\varepsilon)-\Bar{Q}_s\circ F^*_{f^*|s}\circ f^*(X,s)||S=s]\\ &  
 \leq \mathbf{E}\left[\left|\hat{Q}_{\hat{f}|s}(\hat{F}_{\hat{f}|s}(\hat{f}(X,s)+\varepsilon)-Q_{f^*|s}({F}_{f^*|s}(f^*(X,s))\right||S=s \right] \\&+\left|\min\left\{\alpha,\underline{\lambda}_{s}^{-1}\right\}\right|\mathbf{E}\left[ \left|\mathds{1}_{[0,p]}(\hat{F}_{\hat{f}|s}(\hat{f}(X,s)+\varepsilon)-\mathds{1}_{[0,p]}(F_{f^*|s}(f^*(X,s)) \right||S=s\right]\\
  & +\sum_{s'\in\mathcal{S}}p_{s'}\mathbf{E} \left[ \left| \hat{Q}_{\hat{f}|s'}(\hat{F}_{\hat{f}|s'}(\hat{f}(X,s)+\varepsilon)-\sum_{s'\in\mathcal{S}}Q_{f^*|s'}(F_{f^*|s'}({f}^*(X,s))\right| |S=s\right]\\
 &+\max\left\{\alpha+\xi,\max_{s\in\mathcal{S}}\underline{\lambda}_{s^{\prime}}^{-1}\right\}\mathbf{E}\left[ \left|\mathds{1}_{(p,1])}(\hat{F}_{\hat{f}|s}(\hat{f}(X,s)+\varepsilon)-\mathds{1}_{(p,1]}(F_{f^*|s}(f^*(X,s)) \right||S=s\right]
\end{align*}

Now, we bound the first term in the {\it r.h.s.} of the above inequlity.
Using same arguments as in the proof of Theorem~4.4 in~\citep{chzhen2020fairwb},
we obtain
\begin{equation}
\label{eq:eqRisk1}
   \mathbf{E}\left[\left|\hat{Q}_{\hat{f}|s'}(\hat{F}_{\hat{f}|s}(\hat{f}(X,s)+\varepsilon)-Q_{f^*|s'}({F}_{f^*|s}(f^*(X,s))\right||S=s \right] \leq C \left(\sqrt{\dfrac{1}{N_s}}+ b_n^{-1/2}\right).
\end{equation}


To finish the proof, it remains to control
\begin{equation*}
{\mathbf{E}\left[ \left|\mathds{1}_{(p,1])}(\hat{F}_{\hat{f}|s}(\hat{f}(X,s)+\varepsilon)-\mathds{1}_{(p,1]}(F_{f^*|s}(f^*(X,s)) \right||S=s\right]}. 
\end{equation*}
First, we focus on the control conditional on the first sample $\mathcal{D}_n$ of
\begin{equation*}
\left| \hat{F}_{\hat{f}|s}(\hat{f}(X,s)+\varepsilon )-F_{f^*|s}(f^*(X,s)) \right|.    
\end{equation*}
The following decomposition holds
\begin{multline*}
 \hat{F}_{\hat{f}|s}(\hat{f}(X,s)+\varepsilon )-F_{f^*|s}(f^*(X,s)) = 
 \hat{F}_{\hat{f}|s}(\hat{f}(X,s)+\varepsilon)- {F}_{\hat{f}|s}(\hat{f}(X,s)+\varepsilon ) \\ + {F}_{\hat{f}|s}(\hat{f}(X,s)+\varepsilon )- F_{f^*|s}(\hat{f}(X,S)+\varepsilon) + F_{f^*|s}(\hat{f}(X,S)+\varepsilon)-F_{f^*|s}(f^*(X,s)).
\end{multline*}
We note that Assumption~\ref{as1} ensures that $F_{f^*|s}$ is $\bar{\lambda}_s$ Lipschitz.
Therefore, from the above decomposition, we  deduce that
\begin{multline*}
\left| \hat{F}_{\hat{f}|s}(\hat{f}(X,s)+\varepsilon )-F_{f^*|s}(f^*(X,s)) \right|
\leq \sup_{t \in \mathbb{R}} \left|\hat{F}_{\hat{f}|s}(t)- {F}_{\hat{f}|s}(t)\right| +
\bar{\lambda}_s \left| \hat{f}(X,S)+\varepsilon-f^*(X,S)\right| +\\ \left|{F}_{\hat{f}|s}(\hat{f}(X,s)+\varepsilon )- F_{f^*|s}(\hat{f}(X,S)+\varepsilon)\right|.  
\end{multline*}
Besides, from Lemma~\ref{lem:lemRiskCDF} we have that
\begin{equation*}
 \left|{F}_{\hat{f}|s}(\hat{f}(X,s)+\varepsilon )- F_{f^*|s}(\hat{f}(X,S)+\varepsilon)\right| \leq \left\|\hat{f}-f^*\right\|_{\infty} + \varepsilon.  
\end{equation*}
Hence, we deduce
\begin{equation*}
\left| \hat{F}_{\hat{f}|s}(\hat{f}(X,s)+\varepsilon )-F_{f^*|s}(f^*(X,s)) \right|
\leq C\left( \sup_{t \in \mathbb{R}} \left|\hat{F}_{\hat{f}|s}(t)- {F}_{\hat{f}|s}(t)\right| +  
\left\|\hat{f}-f^*\right\|_{\infty} + \varepsilon \right):=\Gamma_{N,n}.
\end{equation*}
Therefore, from the above inequality, we get
\begin{multline*}
\mathbf{E}\left[ \left|\mathds{1}_{(p,1])}(\hat{F}_{\hat{f}|s}(\hat{f}(X,s)+\varepsilon)-\mathds{1}_{(p,1]}(F_{f^*|s}(f^*(X,s)) \right||S=s\right]
\leq \\ \mathbf{E}\left[ \mathds{1}_{\left\{ \left| \hat{F}_{\hat{f}|s}(\hat{f}(X,s)+\varepsilon )-F_{f^*|s}(f^*(X,s)) \right|\geq |F_{f^*|s}(f^*(X,s))-p| \right\}}|S=s \right]  \leq \mathbf{E}\left[ \mathds{1}_{\left\{\Gamma_{n,N}\geq |F_{f^*|s}(f^*(X,s))-p| \right\}}|S=s \right] \\
\end{multline*}
Since for each $s \in \mathcal{S}$, $F_{f^*|s}$
is continuous and we have that $F_{f^*|s}(f^*(X,s))$
is distributed according to a Uniform distribution on $[0,1]$. Hence, conditional on $\mathcal{D}_n$, and $\mathcal{D}_N$, we have that
\begin{equation*}
\mathbb{P}_{X|S=s}\left(|F_{f^*|s}(f^*(X,s))-p| \leq \Gamma_{n,N} \right)   \leq C \Gamma_{n,N}.
\end{equation*}
Therefore, we deduce
\begin{equation*}
\mathbf{E}\left[ \left|\mathds{1}_{(p,1]}(\hat{F}_{\hat{f}|s}(\hat{f}(X,s)+\varepsilon)-\mathds{1}_{(p,1]}(F_{f^*|s}(f^*(X,s)) \right||S=s\right]
\leq  C\mathbb{E}\left[\sup_{t \in \mathbb{R}} \left|\hat{F}_{\hat{f}|s}(t)- {F}_{\hat{f}|s}(t)\right| +  
\left\|\hat{f}-f^*\right\|_{\infty} +\varepsilon \right].     
\end{equation*}
From DKW inequality, and Assumption~\ref{as1}, we get
\begin{equation}
\label{eq:eqRisk2}
\mathbf{E}\left[ \left|\mathds{1}_{(p,1])}(\hat{F}_{\hat{f}|s}(\hat{f}(X,s)+\varepsilon)-\mathds{1}_{(p,1]}(F_{f^*|s}(f^*(X,s)) \right||S=s\right]
\leq C\sum_{s \in \mathcal{S}}\left(\sqrt{\dfrac{1}{N_s}}+ \dfrac{\log(n)}{b_n^{1/2}}\right).   
\end{equation}


Similarly,
\begin{align*}
     \mathbf{E}\left[ \left|\mathds{1}_{[0,p]}(\hat{F}_{\hat{f}|s}(\hat{f}(X,s)+\varepsilon)-\mathds{1}_{[0,p]}(F_{f^*|s}(f^*(X,s)) \right||S=s\right]  & \leq 
C\left(\sqrt{\dfrac{1}{N_s}}+ \dfrac{\log(n)}{b_n^{1/2}}\right). 
\end{align*}
Thus, From the above inequality, Equation~\eqref{eq:eqRisk1}, ~\eqref{eq:eqRisk2}

\begin{equation*}
    \mathbf{E}[|\hat{Q}_s\circ \hat{F}_{\hat{f}|s}\circ (\hat{f}(X,s)+\varepsilon)-\Bar{Q}_s\circ F^*_{f^*|s}\circ f^*(X,s)||S=s] \leq C\sum_{s \in \mathcal{S}}\left(\sqrt{\dfrac{1}{N_s}}+ \dfrac{\log(n)}{b_n^{1/2}}\right).
\end{equation*}
Finally the above inequality, and  Equation~\eqref{eq:eqRisk0}, together with Lemma~4.1 in~\citep{GyorfyBook2002}
yields the desired result.
   
\end{proof}

\section{Additional numerical experiments}
\label{app:additional numeric}

This appendix presents additional illustrations on \texttt{CRIME},  \texttt{California Housing}, and  \texttt{Law school} datasets focusing on CDF adjustment obtained thanks to $\hat{g}_{\alpha,p}^{\xi}$.

\begin{figure}[t]
  \centering
\includegraphics[width=\linewidth]{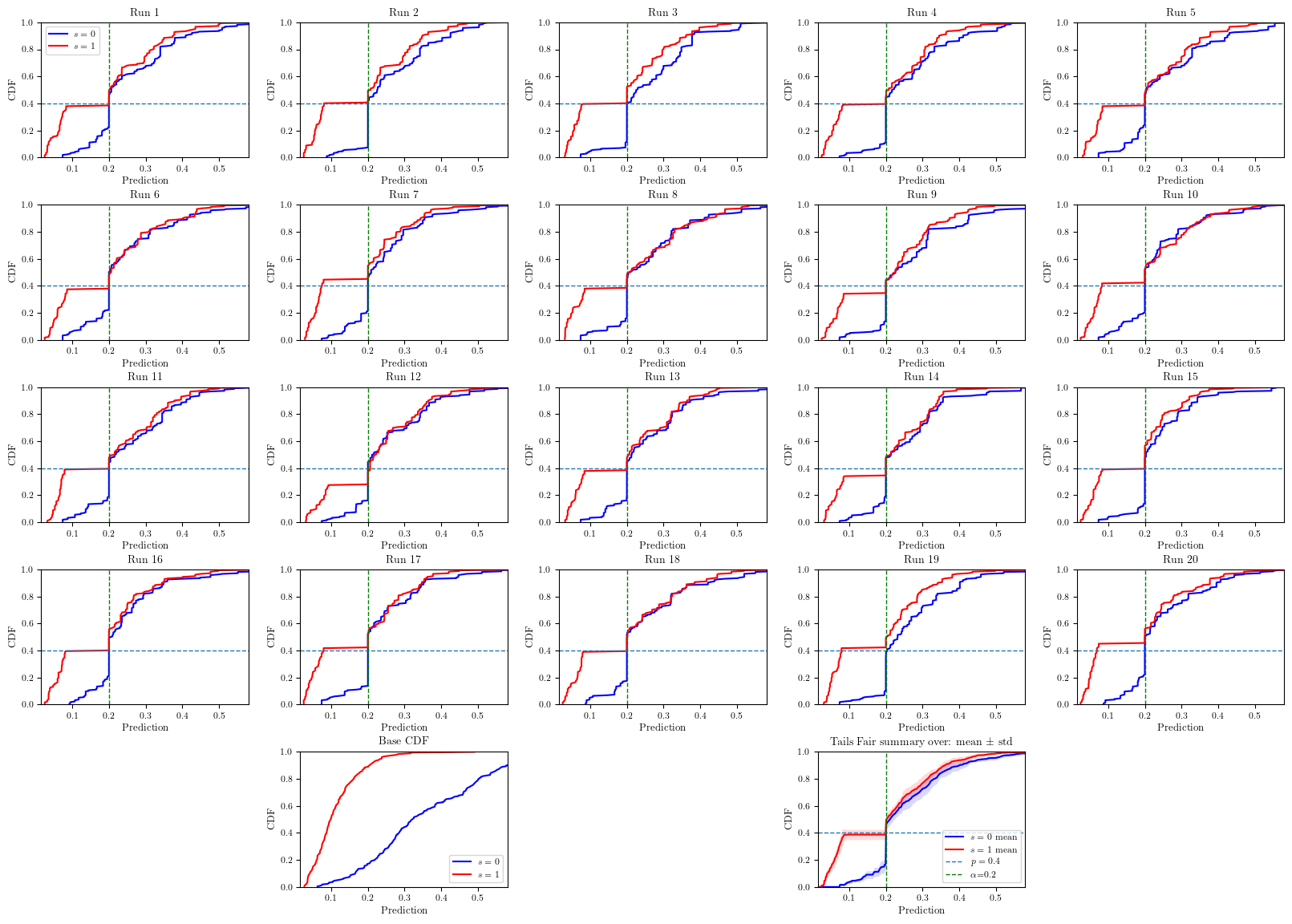}\hfill
    \caption{Empirical CDF of the predictions on the \texttt{Crime} before and after enforcing DP-tails fairness over 20 runs. The parameters are set (arbitrary) to $\alpha=0.2$, $p=0.4$. }
\label{fig:sumcrime}
\end{figure}

\begin{figure}[t]
  \centering
\includegraphics[width=\linewidth]{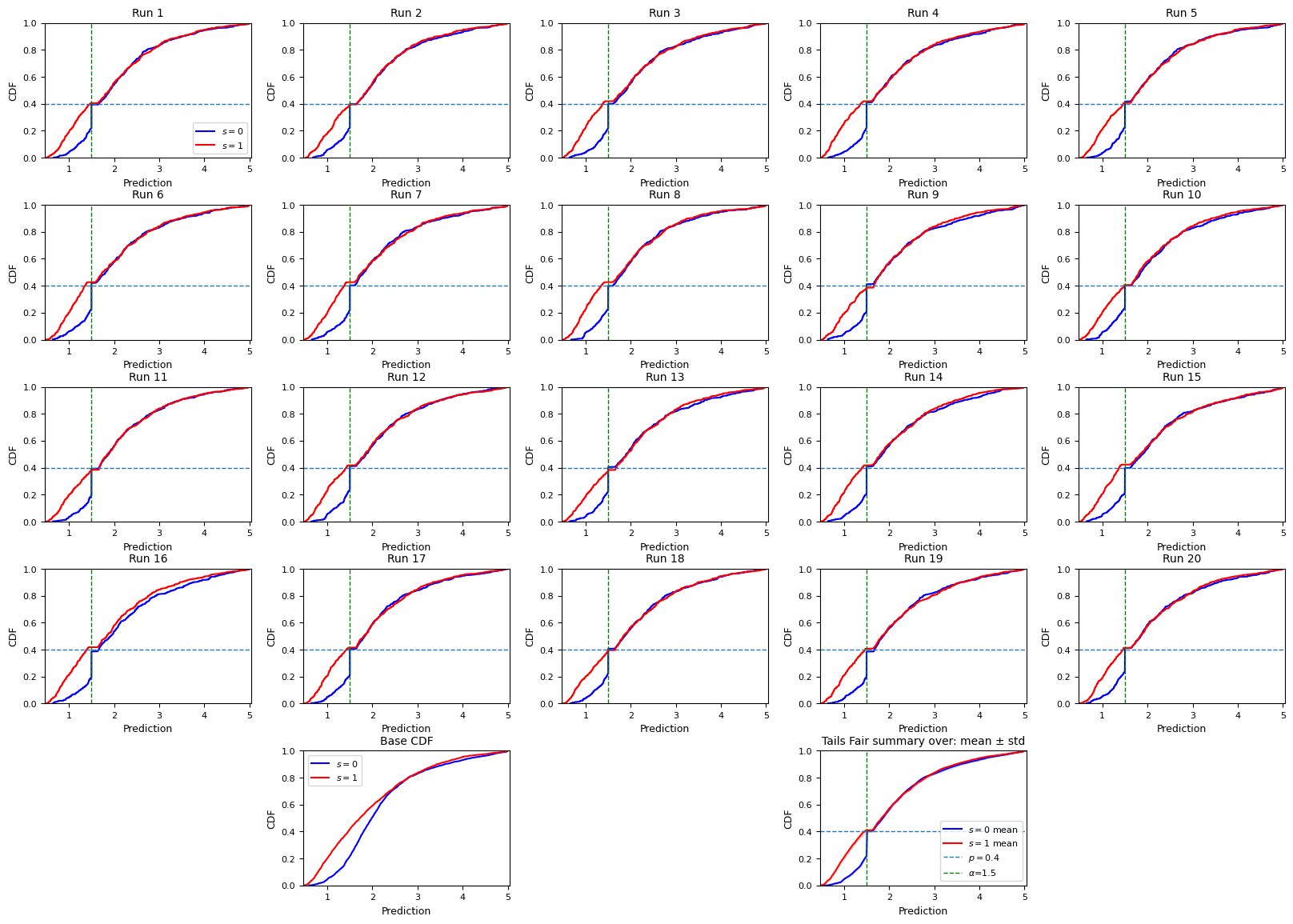}\hfill
    \caption{Empirical CDF of the predictions on the \texttt{California Housing} before and after enforcing DP-tails fairness over 20 runs. The parameters are set (arbitrary) to $\alpha=1.5$, $p=0.4$. }
\label{fig:sumcali}
\end{figure}

\begin{figure}[t]
  \centering
\includegraphics[width=\linewidth]{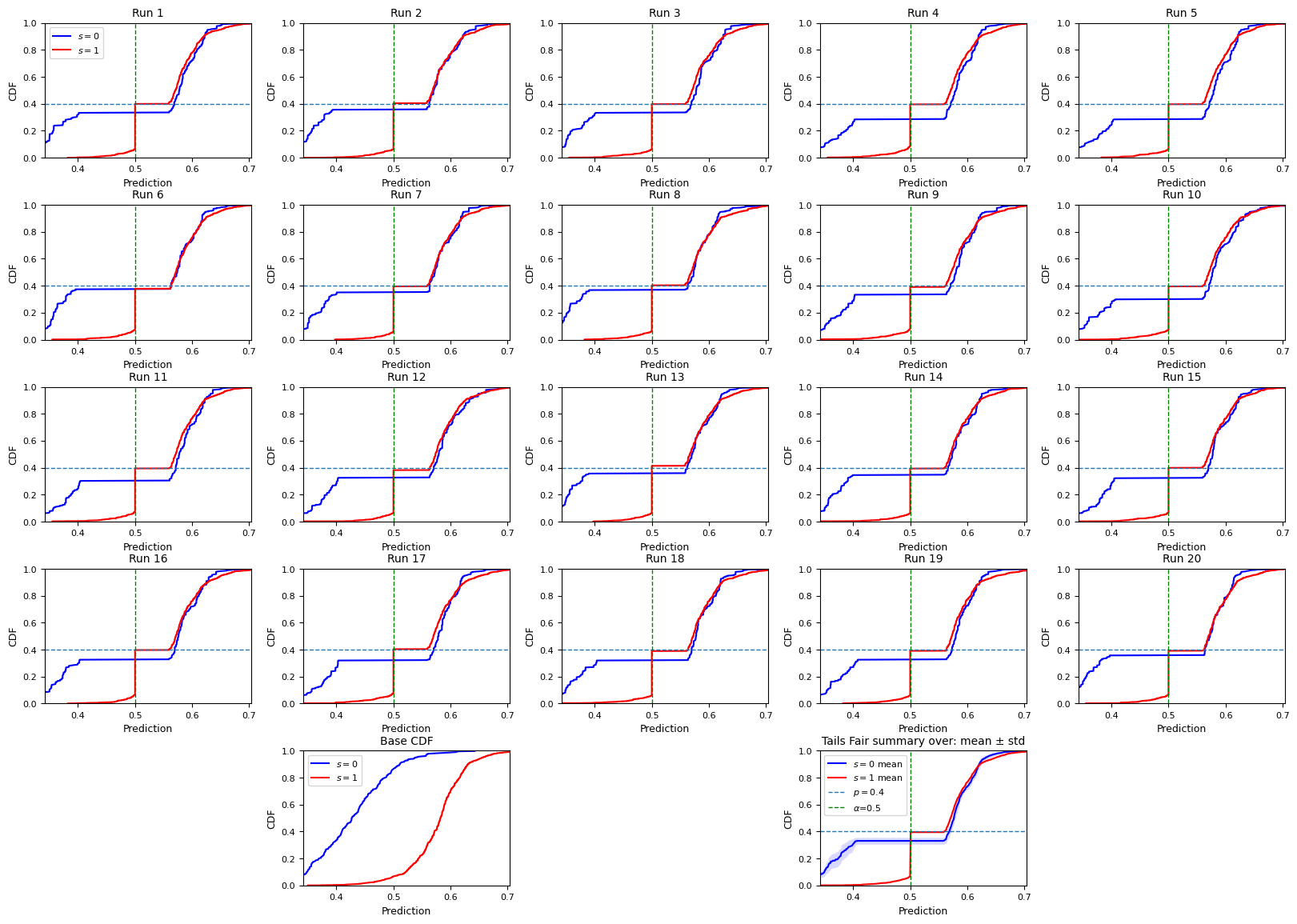}\hfill
    \caption{Empirical CDF of the predictions on the \texttt{Law school} before and after enforcing DP-tails fairness over 20 runs. The parameters are set (arbitrary) to $\alpha=0.5$, $p=0.4$. }
\label{fig:sumlaw}
\end{figure}
\end{document}